\documentclass{article}

\usepackage[numbers]{natbib}
\usepackage[preprint]{neurips_2019}

\usepackage[utf8]{inputenc} 
\usepackage[T1]{fontenc}    
\usepackage{hyperref}       
\usepackage{url}            
\usepackage{booktabs}       
\usepackage{amsfonts}       
\usepackage{nicefrac}       
\usepackage{microtype}   
\usepackage{lipsum}
\usepackage{enumitem}

\usepackage{thm-restate}
\usepackage{multirow}
\usepackage{multicol}
\usepackage{subcaption}
\usepackage{wrapfig}
\usepackage[dvipsnames]{xcolor}

\usepackage{amsmath,amsfonts,bm}

\def\eqref#1{Eq.~(\ref{#1})}
\def\Eqref#1{Equation~\ref{#1}}

\def\1{\bm{1}}

\DeclareMathAlphabet{\mathsfit}{\encodingdefault}{\sfdefault}{m}{sl}
\SetMathAlphabet{\mathsfit}{bold}{\encodingdefault}{\sfdefault}{bx}{n}

\def\gC{{\mathcal{C}}}

\def\gE{{\mathcal{E}}}
\def\gF{{\mathcal{F}}}
\def\gG{{\mathcal{G}}}
\def\gH{{\mathcal{H}}}

\def\gK{{\mathcal{K}}}

\def\gN{{\mathcal{N}}}
\def\gO{{\mathcal{O}}}

\def\gR{{\mathcal{R}}}

\usepackage{amsmath,amsfonts,bm,amsthm}

\usepackage{amsfonts}
\theoremstyle{definition}
\newtheorem{definition}{Definition}[section]
\usepackage{color}
\usepackage{Definitions}
\usepackage{graphicx}
\newcommand*{\Scale}[2][4]{\scalebox{#1}{$#2$}}%

\usepackage{algorithmic}
\usepackage[vlined,ruled]{algorithm2e}

\usepackage{caption}
\newenvironment{Figure}
  {\par\medskip\noindent\minipage{\linewidth}}
  {\endminipage\par\medskip}
\theoremstyle{remark}
\newtheorem*{remark}{Remark}

\hypersetup{
colorlinks=true,
}

\title{Can Graph Neural Networks Help Logic Reasoning?}

\author{
\hspace{-3mm}
Yuyu Zhang$^*$, Xinshi Chen$^*$, Yuan Yang$^*$, Arun Ramamurthy$^\dag$, Bo Li$^\ddagger$, Yuan Qi$^\diamond$, Le Song$^{*\diamond}$\\
$^*$Georgia Tech, $^\dag$Siemens, $^\ddagger$UIUC, $^\diamond$Ant Financial\\
}

\begin{document}
\maketitle

\begin{abstract}
Effectively combining logic reasoning and probabilistic inference has been a long-standing goal of machine learning: the former has the ability to generalize with small training data, while the latter provides a principled framework for dealing with noisy data. However, existing methods for combining the best of both worlds are typically computationally intensive. In this paper, we focus on Markov Logic Networks and explore the use of graph neural networks (GNNs) for representing probabilistic logic inference. It is revealed from our analysis that the representation power of GNN alone is not enough for such a task. We instead propose a more expressive variant, called ExpressGNN, which can perform effective probabilistic logic inference while being able to scale to a large number of entities. We demonstrate by several benchmark datasets that ExpressGNN has the potential to advance probabilistic logic reasoning to the next stage.
\end{abstract}

\vspace{-3mm}
\section{Introduction}

\setlength{\abovedisplayskip}{3pt}
\setlength{\abovedisplayshortskip}{3pt}
\setlength{\belowdisplayskip}{3pt}
\setlength{\belowdisplayshortskip}{3pt}
\setlength{\jot}{2pt}
\setlength{\floatsep}{1ex}
\setlength{\textfloatsep}{1ex}
\setlength{\intextsep}{1ex}

An elegant framework of combining logic reasoning and probabilistic inference is 
Markov Logic Network (MLN)~\citep{richardson2006markov}, where logic predicates are treated as random variables and logic formulae are used to define the potential functions.
It has greatly extended the ability of logic reasoning to handle noisy facts and partially correct logic formulae common in
real-world problems. Furthermore, MLN enables probabilistic graphical models to exploit prior knowledge and learn in the region of small or zero training samples. This second aspect is important in the context of lifelong learning and massive multitask learning, where most prediction targets have insufficient number of labeled data. 

\begin{wrapfigure}[17]{R}{0.331\textwidth}
    \vspace{-2.5mm}
    \centering
    \includegraphics[width=0.33\textwidth]{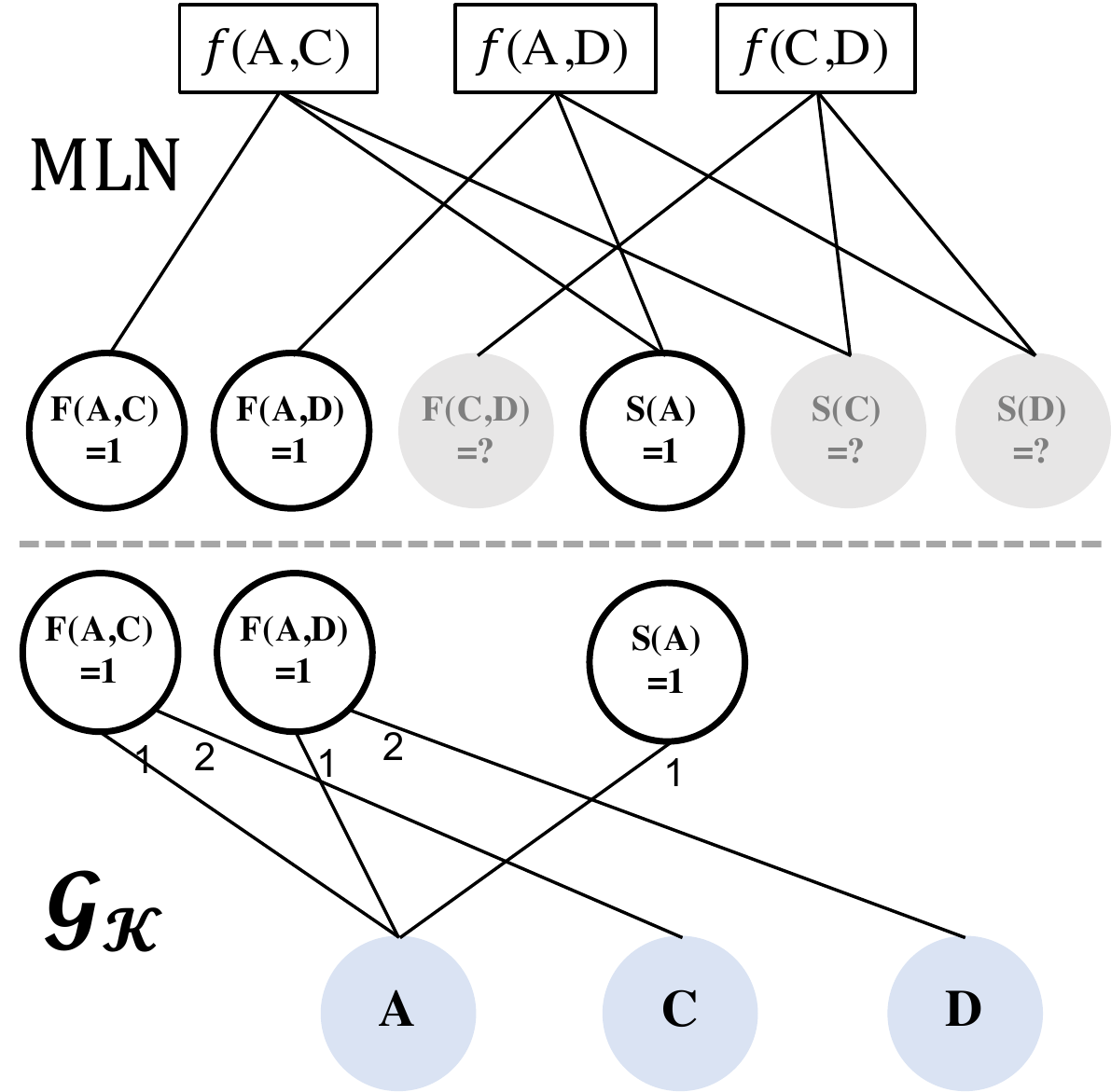}
    \vspace{-4.5mm}
    \caption{\small {\it Bottom}: A knowledge base as a factor graph. $\Scale[0.95]{\{A,C,D\}}$ are entities, and \texttt{F} (\texttt{Friend}) and \texttt{S} (\texttt{Smoke}) are predicates. {\it Top}: Markov Logic Network (MLN) with formula $f(c,c'):=\lnot \texttt{S}(c) \lor \lnot \texttt{F}(c,c') \lor \texttt{S}(c')$.}
    \vspace{-1mm}
    \label{fig:factor_graph}
\end{wrapfigure}
However, a central challenge is that probabilistic inference in MLN is computationally intensive. It contains $\gO(M^n)$ many random variables if there are $M$ entities and the involved predicate has $n$ arguments.
Approximate inference techniques such as MCMC and belief propagation have been proposed, but the large MLN makes them barely scalable to hundreds of entities.

Graph neural network (GNN) is a popular tool of learning representation for graph data, including but not limited to social networks, molecular graphs, and knowledge graphs~\citep{bruna2014spectral,duvenaud2015convolutional,dai2016discriminative,li2016gated,kipf2017semi,hamilton2017inductive}. It is natural to think that GNNs have the potential to improve the effectiveness of probabilistic logic inference in MLN. However, it is not clear \textit{why} and \textit{how} exactly GNNs may help.

In this paper, we explore the use of GNN for scalable probabilistic logic inference in MLN, and provide an affirmative answer on how to do that. In our method,
GNN is applied to knowledge bases which can be orders of magnitude smaller than grounded MLN; and then GNN embeddings are used to define mean field distributions in probabilistic logic inference. However, our analysis reveals that GNN embeddings alone will lead to {\bf inconsistent} parametrization due to the additional asymmetry created by logic formulae. Motivated by this analysis, we propose a more expressive variant, called ExpressGNN, which consists of {\bf (1)} an inductive GNN embedding component for learning representation from knowledge bases; {\bf (2)} and a transductive and tunable embedding component for compensating the asymmetry created by logic formulae in MLN.

We show by experiments that mean field approximation with ExpressGNN
enables efficient and effective probabilistic logic inference in modern knowledge bases. Furthermore, ExpressGNN can achieve these results with far fewer parameters than purely transductive embeddings, and yet it has the ability to adapt and generalize to new knowledge graphs.

{\bf Related work.} Previous probabilistic logic inference techniques either use sampling methods or belief propagation. More advanced variants have been proposed to make use of symmetries in MLNs to reduce computation (e.g., the lifted inference algorithms~\citep{singla2008lifted,singla2014approximate}).
However, these inference methods still barely scale to hundreds of entities. We describe specific methods and compare their performances in Section~\ref{sec:exp}. A recent seminal work explored the use of GNN for relation prediction, but the additional challenges from logic formula are not considered~\citep{qu2019gmnn}.
\vspace{-1mm}
\section{Knowledge Bases and Markov Logic Networks}
\vspace{-1mm}

\textbf{Knowledge base.} Typically, a knowledge base $\Kcal$ consists of a tuple $\Kcal = (\Ccal, \Rcal, \Ocal)$, with a set $\Ccal=\{c_1,\ldots,c_M\}$ of $M$ entities, a set $\Rcal=\{r_1,\ldots,r_N\}$ of $N$ relations, and a collection $\Ocal=\{o_1,\ldots,o_L\}$ of $L$ observed facts.  In the language of first-order logic, 
entities are also called constants. For instance, a constant can be a person or an object.  Relations are also called {\bf predicates}. Each predicate is a logic function defined over $\Ccal$, i.e.,
  $
    r(\cdot):\Ccal \times \ldots \times \Ccal\mapsto \cbr{0,1}.
  $
In general, the arguments of predicates are asymmetric.
  For instance, for the predicate $r(c,c'):=\texttt{L}(c,c')$ (\texttt{L} for \texttt{Like}) which checks whether $c$ likes $c'$, the arguments $c$ and $c'$ are not exchangeable.

With a particular set of entities assigned to the arguments, the predicate is called a grounded predicate, and {\bf each grounded predicate $\equiv$ a binary random variable},
which will be used to define MLN.
For a $d$-ary predicate, there are $M^d$ ways to ground it.
  We denote an assignment as $a_r$. For instance, with $a_r = (c,c')$, we can simply write a grounded predicate $r(c,c')$ as $r(a_r)$.
 Each observed fact in knowledge bases is a truth value $\cbr{0,1}$ assigned to a grounded predicate. For instance, a fact $o$ can be $[\texttt{L}(c,c')=1]$. 
The number of observed facts is typically much smaller than that of {unobserved facts}. We adopt the {open-world} paradigm and treat these {\bf unobserved facts $\equiv$ latent variables}.

As a more clear representation, we express a knowledge base $\gK$ by a bipartite graph $\Gcal_{\Kcal}=(\Ccal,\Ocal,\Ecal)$, where nodes on one side of the graph correspond to constants $\Ccal$ and nodes on the other side correspond to observed facts $\Ocal$, which is called \textit{factor} in this case (Fig.\ref{fig:factor_graph}). The set of $T$ edges, $\Ecal=\cbr{e_1,\ldots,e_T}$, will connect constants and the observed facts. More specifically, an edge
\begin{quote}
    \vspace{-3mm}
    $e = (c, o, i)$ between node $c$ and $o$ exists, if the grounded predicate associated with $o$ uses $c$ as an argument in its $i$-th argument position. (See Fig.~\ref{fig:factor_graph} for an illustration.)    
\end{quote}
\vspace{-2mm}
\textbf{Markov Logic Networks.} MLNs use logic formulae to define potential functions in undirected graphical models.  A logic formula $f(\cdot):\Ccal\times\ldots\times\Ccal\mapsto \{0,1\}$ is a binary function defined via the composition of a few predicates. For instance, a logic formula $f(c,c')$ can be
  \begin{align*}
    \texttt{Smoke}(c) \land \texttt{Friend}(c,c') \Rightarrow \texttt{Smoke}(c') \quad \Longleftrightarrow \quad \lnot \texttt{Smoke}(c) \lor \lnot \texttt{Friend}(c,c') \lor \texttt{Smoke}(c'), 
  \end{align*}
  where $\lnot$ is negation and the equivalence is established by De Morgan's law.
  Similar to predicates, %
  we denote an assignment of constants to the arguments of a formula $f$ as $a_f$, and the entire collection of consistent assignments of constants as $\Acal_f=\{a_f^1,a_f^2,\ldots\}$. Given these logic representations, 
MLN can be defined as a joint distribution over all observed facts $\Ocal$ and unobserved facts $\Hcal$ as (Fig.~\ref{fig:factor_graph})
\begin{align}
    \label{eq:mln}
    P\rbr{\Ocal, \Hcal} 
    := {\textstyle \frac{1}{Z} \exp\rbr{\sum\nolimits_{f\in \Fcal}w_f \sum\nolimits_{a_f \in \Acal_f} \phi_f(a_f)} },
\end{align}
where $Z$ is a normalizing constant summing over all grounded predicates and $\phi_f(\cdot)$ is the potential function defined by a formula $f$.
One form of $\phi_f(\cdot)$ can simply be the truth value of the logic formula $f$. For instance, if the formula is $f(c,c'):=\lnot \texttt{S}(c) \lor \lnot \texttt{F}(c,c') \lor \texttt{S}(c')$, then $\phi_f(c,c')$ can simply take value $1$ when $f(c,c')$ is true and $0$ otherwise. Other more sophisticated $\phi_f$ can also be designed, which have the potential to take into account complex entities, such as images or texts, but will not be the focus of this paper. The weight $w_f$ can be viewed as the confidence score of formulae $f$: the higher the weight, the more accurate the formula is.%

\section{Challenges for Inference in Markov Logic Networks}
\vspace{-0.5mm}
Inference in Markov Logic Networks can be very computationally intensive, since the inference needs to be carried out in the fully grounded network involving all grounded variables and formula nodes. Most previous inference methods barely scale to hundreds of entities.

{\bf Mean field approximation.} We will focus on mean field approximation, since it has been demonstrated to scale up to many large graphical models, such as latent Dirichlet allocation for modeling topics from large text corpus~\citep{qu2019gmnn,hoffman2013stochastic}. In this case, the conditional distribution $P(\Hcal|\Ocal)$ is approximated by a product distribution, 
$P(\Hcal | \Ocal) \approx \prod_{r(a_r) \in \Hcal} Q^*(r(a_r))$. 
The set of mean field distributions $Q^*(r(a_r))$ can be determined by KL-divergence minimization 
\begin{align}
\hspace{-2mm}
    \cbr{Q^*(r(a_r))} 
    &= \text{argmin}_{\{Q(r(a_r))\}}~~  
    \Scale[0.95]{\text{KL}\rbr{\prod\nolimits_{r(a_r) \in \Hcal} Q(r(a_r)) \quad \| \quad P(\Hcal | \Ocal)} \label{eq:meanfield_obj}}    \\
    &= \text{argmin}_{\{Q(r(a_r))\}}~~
    \Scale[0.95]{\sum\nolimits_{r(a_r)\in\Hcal}\EE[\ln Q(r(a_r))] - \sum\nolimits_{f\in \Fcal}w_f\sum\nolimits_{a_f\in  \Acal_f}\EE[\phi_f(a_f)|\Ocal]}, \nonumber
\end{align}
where $\EE[\phi_f(a_f)|\Ocal]$ means that observed predicates in $\Ocal$ are fixed to their actual values with probability 1. For instance, a grounded formula $f(A,B)=\lnot \texttt{S}(A) \lor \lnot \texttt{F}(A,B) \lor \texttt{S}(B)$ with observations $\texttt{S}(A)=1$ and  $\texttt{F}(A,B)=1$ will result in $\EE[\phi_f(a_f)|\Ocal]=\sum_{\texttt{S}(B)=\{0,1\}}Q(\texttt{S}(B))(\lnot 1 \lor \lnot 1 \lor \texttt{S}(B))$. In theory, one can use mean field iteration to obtain the optimal $Q^*(r(a_r))$, but due to the large number of nodes in the grounded network, this iterative algorithm can be very inefficient.

Thus, we need to carefully think about how to parametrize the set of $Q(r(a_r))$, such that the parametrization is expressive enough for representing posterior distributions, while at the same time leading to efficient algorithms. Some common choices are
\begin{itemize}[leftmargin=*,nolistsep,nosep]
    \item {\bf Naive parametrization.} Assign each $Q(r(a_r))$ a parameter $q_{r(a_r)}\Scale[0.9]{\in[0,1]}$. Such parametrization is very expressive, but the number of parameters is the same as MLN size $\gO(M^n)$.
    \item {\bf Tunable embedding parametrization.} Assign each entity $c$ a vector embedding $\mu_c\in \RR^d$, and define $Q(r(a_r))$ using involved entities. For instance, $Q(r(c,c')):=\texttt{logistic}\big(\texttt{MLP}_r(\mu_c,\mu_c')\big)$ where $\texttt{MLP}_r$ is a neural network specific to predicate $r$ and $\texttt{logistic}(\cdot)$ is the standard logistic function. The number parameters in such scheme is linear in the number of entities, $\Ocal(d|\Ccal|)$, but very high dimensional embedding $\mu_c$ may be needed to express the posteriors. 
\end{itemize}

\begin{wrapfigure}[22]{r}{0.45\textwidth}
\vspace{-3mm}
    \begin{algorithm}[H]
      \DontPrintSemicolon
      \SetKwFunction{Grad}{Grad}
      \SetKwProg{Fn}{Function}{:}{}
      \SetKwFor{uFor}{For}{do}{}
      \SetKwFor{ForPar}{For all}{do in parallel}{}
      \SetKwComment{Comment}{$\triangleright$\ }{}
      \SetCommentSty{mycommfont}
      \SetKwFunction{GNN}{GNN}
    \Fn{\GNN{$\gG_{\gK}=(\gC,\gO,\gE)$}}{
      Initialize entity node: $\mu_c^{(0)} = \mu_{\gC},~\forall c \in \Ccal$\;
      Fact node: $\mu_o^{(0)} = \mu_r,\forall o\equiv\sbr{r(a_r)=v}$\;
      \Comment{nodes of the same type are initialized with a uniform color}
      \uFor{$t=0$ to $T-1$}{
      Compute message $\forall (c,o,i)\in\gE$: \;
      $m_{o\rightarrow c}^{(t+1)} = \texttt{MLP}_{1,i,v}(\mu_o^{(t)}, \mu_c^{(t)})$\;
      $m_{c\rightarrow o}^{(t+1)} = \texttt{MLP}_{2,i,v}(\mu_c^{(t)}, \mu_o^{(t)})$\;
      Aggregate message $\forall c\in\gC,o\in\gO$: \;
      $m_c^{(t+1)} = \texttt{AGG}_1(\{m_{o\rightarrow c}^{(t+1)}\}_{o \in \Ncal(c)})$ \;
      $m_o^{(t+1)} = \texttt{AGG}_2(\{m_{c\rightarrow o}^{(t+1)}\}_{c \in \Ncal(o)})$\;
        \Comment{aggregate colors of neighborhoods}
      Update embeddings $\forall c\in\gC,o\in\gO$:\;
      $\mu_c^{(t+1)} = \texttt{MLP}_3(\mu_c^{(t)}, m_c^{(t+1)})$ \;
      $\mu_o^{(t+1)} = \texttt{MLP}_4(\mu_o^{(t)}, m_o^{(t+1)})$\;
        \Comment{hash colors of nodes and their neighbor- -hoods into unique new colors}
    }}
     \vspace{-2mm}
     \KwRet node embeddings $\{\mu_c^{(T)}\}$ and $\{\mu_o^{(T)}\}$\;
    \vspace{-0.2mm}
    \caption{GNN and {\color{BrickRed} Color Refinement}}\label{algo:gnn}
    \end{algorithm}
\end{wrapfigure}
Note that both schemes are {\bf transductive}, and the learned $q_{r(a_r)}$ or $\mu_c$ can only be used for the training graph, but {can not} be used for new entities or different but related knowledge graphs ({\bf inductive setting}). 

\textbf{Stochastic inference.} The objective in~\eqref{eq:meanfield_obj} contains an expensive summation, making its evaluation and optimization inefficient. For instance, for formula $f(c,c'):=\lnot \texttt{S}(c) \lor \lnot \texttt{F}(c,c') \lor \texttt{S}(c')$, the number of terms involved in $\sum_{a_f \in \Acal_f}$ will be square in the number of entities. Thus, we approximate the objective function with stochastic sampling, and then optimize the parameters in $Q(r(a_r))$ via stochastic gradients, and various strategies can be used to reduce the variance of the stochastic gradients~\citep{qu2019gmnn,hoffman2013stochastic}. 

\vspace{-3mm} 
\section{Graph Neural Network for Inference} 
\label{sec:gnn}
\vspace{-2mm} 

To efficiently parametrize the entity embeddings with less parameters than tunable embeddings, we propose to use a GNN on the knowledge graph $\Gcal_{\Kcal}$, much smaller than the fully grounded MLN (Figure~\ref{fig:factor_graph}), to generate embeddings $\mu_c$ of each entity $c$, and then use these embeddings to define mean field distributions. The advantage of GNN is that the number of parameters can be independent of the number of entities. Any entity embedding $\mu_c$ can be reproduced by running GNN iterations online. Thus, GNN based parametrization can potentially be very memory efficient, making it possible to scale up to a large number of entities. Furthermore, the learned GNN parameters can be used for both transductive and inductive settings.

The architecture of GNN over a knowledge graph $\gG_{\gK}$ is given in Algorithm~\ref{algo:gnn},
where the multilayer neural networks $\texttt{MLP}_{1,i,v}$ and $\texttt{MLP}_{2,i,v}$ take values $v$ of observed facts and argument positions $i$ into account, $\texttt{MLP}_3$ and $\texttt{MLP}_4$ are standard multilayer neural networks, and $\texttt{AGG}_1$ and $\texttt{AGG}_2$ are typically sum pooling functions. These embedding updates are carried out for a finite $T$ times. In general, the more iterations are carried out, the larger the graph neighborhood around a node will be integrated into the representation. For simplicity of notation, we use use $\{\mu_c\}$ and $\{\mu_o\}$ to refer to the final embeddings $\{\mu_c^{(T)}\}$ and $\{\mu_o^{(T)}\}$. Then these embeddings are used to define the mean field distributions. For instance, if a predicate $r(c,c')$ involves two entities $c$ and $c'$, $Q(r(c,c'))$ can be defined as
\begin{align}
        Q(r(c,c')) = \texttt{logistic}\rbr{ \texttt{MLP}_r\rbr{\mu_c, \mu_{c'}}},~~\text{where}~\{\mu_c,\mu_o\}=\text{GNN}(\Gcal_{\Kcal})
        \label{eq:para-Q}
\end{align}
and $\texttt{MLP}_r$ are predicate specific multilayer neural networks. For $d$ dimensional embeddings, the number of parameters in GNN model is typically $O(d^2)$, independent of the number of entities. 

\vspace{-3mm}
\section{Is GNN Expressive Enough?}
\vspace{-2mm}

Now a central question is: is GNN parametrization using the knowledge graph $\Gcal_{\Kcal}$ expressive enough? Or will there be situations where two random variables $r(c,c'')$ and $r(c',c'')$ have {\bf different} distributions in MLN but $Q(r(c,c''))$ and $Q(r(c',c''))$ are forced to be {\bf the same} due to the above characteristic of GNN embeddings? 
The question arises since the knowledge graph $\Gcal_{\Kcal}$ is different from the fully grounded MLN (Figure~\ref{fig:factor_graph}). We will use two theorems (proofs are all given in Appendix~\ref{app:thm-proof}) to analyze whether GNN is expressive enough. Our main results can be summarized as:
\begin{itemize}[wide,nolistsep,nosep]
    \item[(1)] Without formulae tying together the predicates, GNN embeddings are expressive enough for feature representations of latent facts in the knowledge base.
    \item[(2)] With formulae in MLN modeling the dependency between predicates, GNN embedding becomes insufficient for posterior parametrization.
\end{itemize}
These theoretical anlayses motivate a more expressive variant, called ExpressiveGNN in Section~\ref{sec:expressgnn}, to compensate for the insufficient representaiton power of GNN for posterior parametrization. 

\vspace{-3mm}
\subsection{Property of GNN}
\vspace{-2mm}

Recent research shows that GNNs can learn to perform approximate graph isomorphism check~\citep{dai2016discriminative,xu2018powerful}. In our case, graph isomorphism of knowledge graphs is defined as follows.
\begin{definition}[Graph Isomorphism] An isomorphism of graphs $\gG_{\gK}=(\gC,\gO,\gE)$ and $\gG_{\gK'}=(\gC',\gO',\gE')$ is a bijection between the nodes $\pi:\gC\cup\gO\rightarrow \gC'\cup\gO'$ such that (1) $\forall o\in \gO,$ $ \texttt{NodeType}\rbr{o}=\texttt{NodeType}\rbr{\pi(o)}$; (2) $c$ and $o$ are adjacent in $\gG_{\gK}$ if and only if $\pi(c)$ and $\pi(o)$ are adjacent in $\gG_{\gK'}$ and $\texttt{EdgeType}\rbr{c,o}= \texttt{EdgeType}\rbr{\pi(c),\pi(o)}$. 
\end{definition}
\vspace{-2mm}
In this definition, $\texttt{NodeType}$ is determined by the associated predicate of a fact, i.e.,
$\texttt{NodeType}(o\equiv[r(a_r)=v])=r$. $\texttt{EdgeType}$ is determined by the observed value and argument position, i.e., $\texttt{EdgeType}\rbr{(c,o\equiv[r(a_r)=v],i) }=(v,i)$. Our GNN architecture in Algorithm~\ref{algo:gnn} is adapted to these graph topology specifications. More precisely, the initial node embeddings correspond to $\texttt{NodeType}$s and different $\texttt{MLP}$s are used for different $\texttt{EdgeType}$s.

It has been proved that GNNs with injections acting on neighborhood features  can be as powerful as 1-dimensional Weisfeiler-Lehman graph isomorphism test~\citep{shervashidze2011weisfeiler,xu2018powerful}, also known as {\bf color refinement}. The color refinement procedure is given in Algorithm~\ref{algo:gnn} by red texts, which is
analogous to updates of GNN embeddings. We use it to define indistinguishable nodes.
\begin{definition}[Indistinguishable Nodes]
Two nodes $c,c'$ in a graph $\gG$ are {\bf indistinguishable} if they have the same color after the color refinement procedure terminates and no further refinement of the node color classes is possible. In this paper, we use the following notation:
\begin{align*}
(c_1,\cdots,c_n) \overset{\gG}{\longleftrightarrow} (c_1',\cdots,c_n'):~\text{for each}~i\in[n],  c_i~\text{and}~c_i'~\text{are indistinguishable in}~\gG.
\end{align*}
\end{definition}
\vspace{-1mm}
While GNN has the ability to do color refinement, it is {\it at most} as powerful as color refinement~\citep{xu2018powerful}. Therefore, if $c\overset{\gG}{\longleftrightarrow}c'$, their final GNN embeddings will be the same, i.e., $\mu_c=\mu_c'$. When we use GNN embeddings to define $Q$ as in \eqref{eq:para-Q}, it implies that $Q(r(c,c''))=Q(r(c',c'')), \forall c'' \in \Ccal$.

\vspace{-3mm}
\subsection{GNN is expressive for feature representations in knowledge bases}
\vspace{-2mm}

GNN embeddings are computed based on knowledge bases ${\Kcal}=(\gC,\gR,\gO)$ involving only observed facts $\Ocal$, but the resulting entity embeddings $\cbr{\mu_c}$ will be used to represent a much larger set of unobserved facts $\gH$. In this section, we will show that, when formulae are not considered, these entity embeddings $\cbr{\mu_c}$  are expressive enough for representing the latent facts $\gH$.

\begin{figure}[h!]
    \centering
    \includegraphics[width=\textwidth]{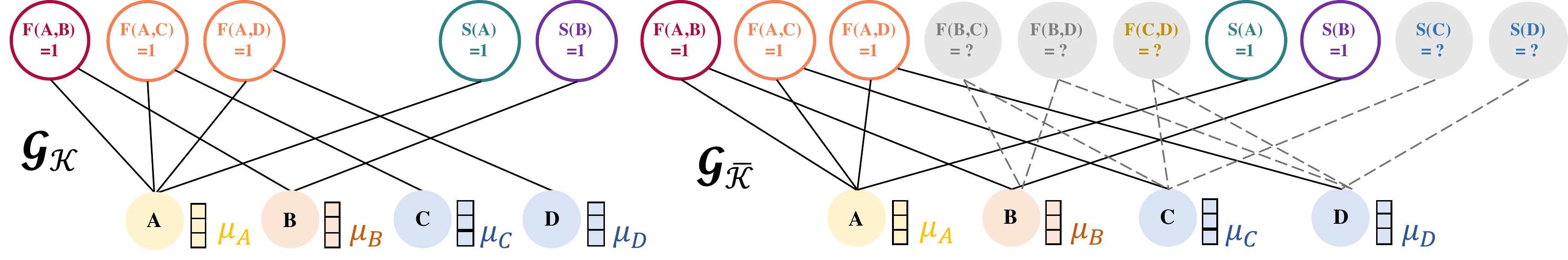}
    \vspace{-3mm}
    \caption{\small Factor graph $\gG_{\gK}$ and the corresponding augmented factor graph $\gG_{\overline{\gK}}$ after running coloring refinement.}
    \label{fig:augmented-graph}
    \vspace{-1mm}
\end{figure}
To better explain our idea, we define a fully grounded knowledge base as:
$
    \overline{\gK}:=\rbr{\gC,\gR,\gO\cup\gH},
$
where all unobserved facts are included and assigned a value of ``$?$''. Therefore, the facts in $\overline{\gK}$ can take one of three different values, i.e., $v\in\cbr{0,1,?}$. Its corresponding {\bf augmented factor graph} is $\gG_{\overline{\gK}}=(\gC,\gO\cup\gH,\gE\cup\gE_{\gH})$.
See Fig.~\ref{fig:augmented-graph} for an illustration of $\Gcal_{\Kcal}$ and $\Gcal_{\overline{\gK}}$.

\begin{restatable}{theorem}{thmAugKnowBase} \label{thm:know-base-unobs}
Let $\gG_{\gK} = (\gC,  \gO,\gE)$ be the factor graph for a knowledge base $\gK$ and $\gG_{\overline{\gK}}=(\gC,  \gO\cup \gH,\gE\cup\gE_{\gH})$ be the corresponding augmented version.
Then the following two statements are true:
\begin{itemize}[wide,nolistsep,nosep]
    \item[(1)] $c\overset{\gG_{\gK}}{\longleftrightarrow} c'$ if and only if $c\overset{\gG_{\overline{\gK}}}{\longleftrightarrow} c'$;
    \item[(2)]  $[r(c_1,\ldots,c_n)=v]\overset{\gG_{\overline{\gK}}}{\longleftrightarrow} [r(c_1',\ldots,c_{n}')=v]$ if and only if $(c_1,\ldots,c_n)\overset{\gG_{\overline{\gK}}}{\longleftrightarrow}(c_1',\ldots,c_n')$.
\end{itemize}
\end{restatable}

Intuitively, Theorem~\ref{thm:know-base-unobs} means that without considering the presence of formulae in MLN, unobserved predicates $\gH$ can be represented purely based on GNN embeddings obtained from $\Gcal_{\Kcal}$. For instance, to represent an unobserved predicate $r(c_1,\ldots,c_n)$ in $\Gcal_{\overline{\gK}}$, we only need to compute $\cbr{\mu_c,\mu_o}=\texttt{GNN}(\gG_{\gK})$, and then use $\texttt{MLP}_r \rbr{\mu_{c_1},\ldots, \mu_{c_{n}}}$ as its feature. This feature representation is as expressive as that obtained from $\cbr{\mu_c,\mu_o}=\texttt{GNN}(\gG_{\overline{\gK}})$,
and thus drastically reducing the computation.

\vspace{-3mm}
\subsection{GNN is not expressive enough for posterior parametrization}\label{sec:counter-example}
\vspace{-2mm}

Taking into account the influence of formulae on the posterior distributions of predicates, we can show that GNN embeddings alone become insufficient representations for parameterizing these posteriors. To better explain our analysis in this section, we first extend the definition of graph isomorphism to node isomorphism and then state the theorem.

\begin{definition}[Isomorphic Nodes]
Two ordered sequences of nodes $(c_1,\ldots,c_n)$ and $(c_1',\ldots,c_n')$ are {\bf isomorphic} in a graph $\gG_{\gK}$ if there exists an isomorphism from $\gG_{\gK}=(\gC,\gO,\gE)$ to itself, i.e., $\pi:\gC\cup\gO \rightarrow \gC\cup\gO$, such that $\pi(c_1)=c_1',\ldots, \pi(c_{n})=c_{n}'$. Further, we use the following notation
\begin{align*}
(c_1,\cdots,c_n) \overset{\gG_{\gK}}{\Longleftrightarrow} (c_1',\cdots,c_n'): (c_1,\cdots,c_n)~\text{and}~(c_1',\cdots,c_n')~\text{are isomorphic in}~\gG_{\gK}.
\end{align*}
\end{definition}
\begin{restatable}{theorem}{thmautomorphism} \label{thm:isomorphism}
    Consider a knowledge base $\gK=(\gC,\gR,\gO)$ and any $r\in\gR$. Two latent random variables $X:=r(c_1,\ldots,c_{n})$ and $X':=r(c_1',\ldots,c_n')$ have the same posterior distribution in {\bf any} MLN {\bf if and only if} $(c_1,\cdots,c_n) \overset{\gG_{\gK}}{\Longleftrightarrow} (c_1',\cdots,c_n')$.
\end{restatable}
\begin{remark}
We say two random variables $X,X'\in\gH$ have the same posterior if the marginal distributions $P(X|\gO)$ and $P(X'|\gO)$ are the same and, moreover, 
for any sequence of random variables $(X_1,\ldots,X_n)$ in $\gH\setminus\cbr{X}$, there exists a sequence of random variables $(X_1',\ldots,X_n')$ in $\gH\setminus\cbr{X'}$ such that the marginal distributions $P(X,X_1,\ldots,X_n|\gO)$ and $P(X',X_1',\ldots,X_n'|\gO)$ are the same.
\end{remark}

A proof is given in Appendix~\ref{app:thm-proof}. The proof of necessary condition is basically showing that, if $(c_1,\ldots,c_{n})$ and $(c_1',\ldots,c_{n}')$ are NOT isomorphic in $\gG_{\gK}$, we can always define a formula which can make $r(c_1,\ldots,c_{n})$ and $r(c_1',\ldots,c_{n}')$ distinguishable in MLN and have different posterior distributions. It implies an important fact that, to obtain an expressive representation for the posterior,
\begin{itemize}[wide,nolistsep,nosep]
    \item[(1)] either GNN embeddings need to be powerful enough to distinguish non-isomorphic nodes;
    \item[(2)] or the information of formulae / MLN need to be incorporated into the parametrization.
\end{itemize}

The second condition (2) somehow defeats our purpose of using mean field approximation to speed up the inference, so it is currently not considered. Unfortunately, condition (1) is also not satisfied, because existing GNNs are {\it at most} as powerful as color refinement, which is not an exact graph isomorphism test. Besides, node isomorphism mentioned in Theorem~\ref{thm:isomorphism} is even more complex than graph isomorphism because it is a constrained graph isomorphism.

We will interpret the implications of this theorem by an example. Figure~\ref{fig:loopy-knowledge-base} shows a factor graph representation for a knowledge base which leads to the following observations:

$\bullet$ Even though $A$ and $B$ have opposite relations with $E$, i.e., $\texttt{F}(A,E)=1$ but $\texttt{F}(B,E)=0$, $A$ and $B$ are {indistinguishable} in $\Gcal_{\Kcal}$ and thus have the same GNN embeddings, i.e., $\mu_A =\mu_B$.

\begin{wrapfigure}[12]{R}{0.355\textwidth}
    \includegraphics[width=0.35\textwidth]{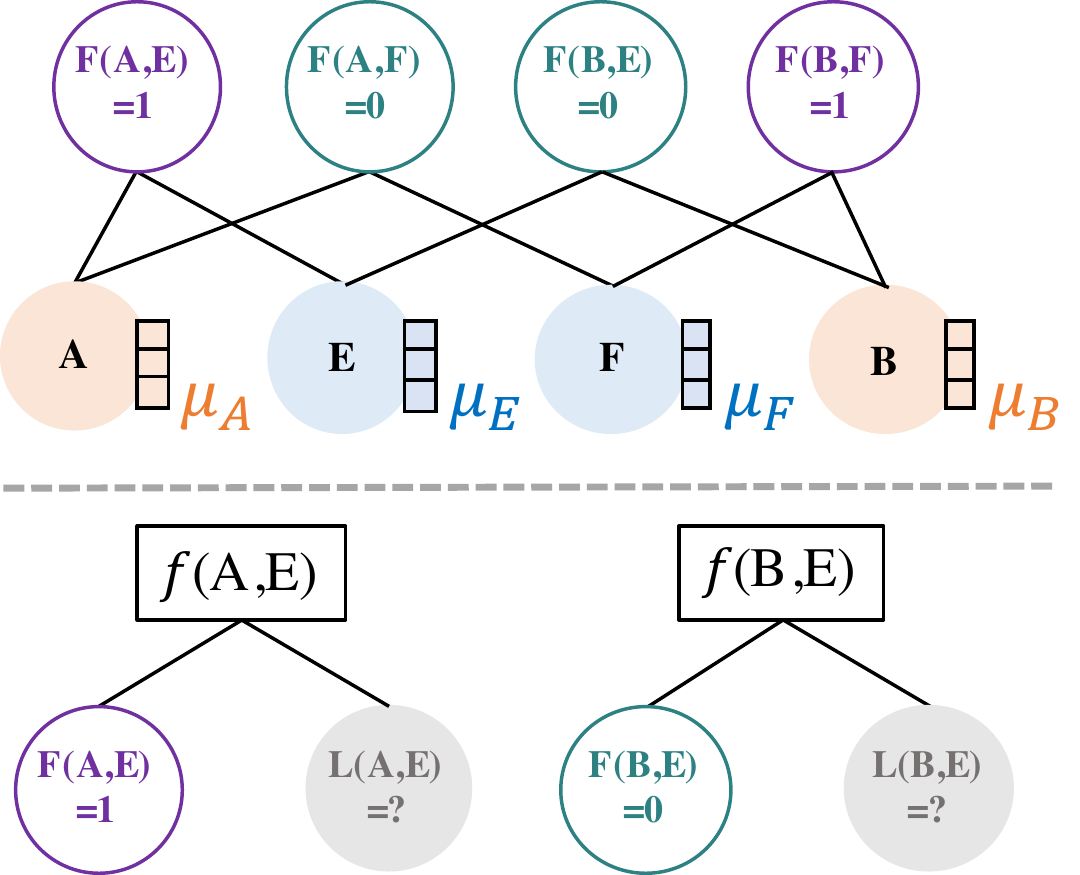}
    \vspace{-1mm}
    \caption{\small {\it Top}: A knowledge base with 0-1-0-1 loop. {\it Bottom}: MLN.}
    \label{fig:loopy-knowledge-base}
\end{wrapfigure}
$\bullet$ Suppose $f(c,c'):=\texttt{F}(c,c')\Rightarrow \texttt{L}(c, c')$ is the only formula in MLN ($\texttt{L}$ is $\texttt{Like}$ and $\texttt{F}$ is $\texttt{Friend}$). $\texttt{L}(A,E)$ and $\texttt{L}(B,E)$ apparently have different posteriors. However, using GNN embeddings, $Q(\texttt{L}(A,E))=\texttt{logistic}\rbr{\texttt{MLP}_{\texttt{L}}(\mu_A,\mu_E)}$ is always identical to $Q(\texttt{L}(B,E))=\texttt{logistic}\rbr{\texttt{MLP}_{\texttt{L}}(\mu_B,\mu_E)}$.

$\bullet$ There exists an isomorphism $\pi_1:(A,E,B,F)\mapsto(B,F,A,$ $E)$ such that $\pi_1(A)=B$, but no isomorphism $\pi$ satisfies both $\pi(A)=B$ and $\pi(E)=E$. Therefore, we see how the isomorphism constraints in Theorem~\ref{thm:isomorphism} make the problem even more complex than graph isomorphism check.

To conclude, it is revealed that node embeddings by GNN alone are not enough to express the posterior in MLN. 
We provide more examples in Appendix \ref{app:counter-examples} to explain that this case is very common and not a rare case. 
In the next section, we will introduce a way of correcting the node embeddings.

\vspace{-1mm}
\section{ExpressGNN: More Expressive GNN with Tunable Embeddings}
\label{sec:expressgnn}
\vspace{-1mm}
\begin{wrapfigure}[4]{R}{0.503\textwidth}
\vspace{-4mm}
\begin{algorithm}[H]
  \DontPrintSemicolon
  \SetKwFunction{Grad}{Grad}
  \SetKwProg{Fn}{Function}{:}{}
  \SetKwFor{uFor}{For}{do}{}
  \SetKwFor{ForPar}{For all}{do in parallel}{}
  \SetKwComment{Comment}{$\triangleright$\ }{}
  \SetCommentSty{mycommfont}
  \SetKwFunction{GNN}{GNN}
$\cbr{\mu_c,\mu_o} \gets$ \GNN{$\gG_{\gK}$}$;$~~~~$\hat{\mu}_c\gets [\mu_c,\omega_c],~\forall c\in\gC;$\;
{\small$Q(r(c_1,\ldots,c_n))$=}$\texttt{logistic}\rbr{ \texttt{MLP}_{r}(\hat{\mu}_{c_1},\ldots,\hat{\mu}_{c_n})}$\;
\caption{$Q(r(c_1,\ldots,c_n))$ with ExpressGNN}\label{algo:q}
\end{algorithm}
\end{wrapfigure}
It is currently challenging to design a new GNN that can check nodes isomorphism, because no polynomial-time algorithm is known even for unconstrained graph isomorphism test~\citep{garey2002computers,babai2016graph}. In this section, we propose a simple yet effective solution. Take Figure~\ref{fig:loopy-knowledge-base} as an example:

To make $Q(\texttt{L}(A,E))$ different from $Q(\texttt{L}(B,E))$, we can simply introduce additional {\bf low dimensional} tunable embeddings $\omega_A$, $\omega_B$, and $\omega_E$ and correct the parametrization as
\begin{align*}
   \texttt{logistic}\rbr{ \texttt{MLP}_{\texttt{L}}([\mu_A,\omega_A],[\mu_E,\omega_E])}~\text{and}~
    \texttt{logistic}\rbr{ \texttt{MLP}_{\texttt{L}}([\mu_B,\omega_B],[\mu_E,\omega_E])}.
\end{align*}
With tunable $\omega_A$ and $\omega_B$,  $Q(\texttt{L}(A,E))$ and $Q(\texttt{L}(B,E))$ can result in different values. In general, we can assign each entity $c\in\gC$ a low-dimensional tunable embedding $\omega_c$ and concatenate it with the GNN embedding $\mu_c$ to represent this entity. We call this variant ExpressGNN and describe the parametrization of $Q$ in Algorithm~\ref{algo:q}.

One can think of ExpressGNN as a hierarchical encoding of entities: GNN embeddings assign similar codes to nodes similar in knowledge graph neighborhoods, while the tunable embeddings provide additional capacity to code variations beyond knowledge graph structures. The hope is %
only a very low dimensional tunable embedding is needed to fine-tune individual differences. Then the total number of parameters in ExpressGNN could be much smaller than using tunable embedding alone. 

ExpressGNN also presents an interesting trade-off between induction and transduction ability. The GNN embedding part allows ExpressGNN to possess some generalization ability to new entities and different knowledge graphs; while the tunable embedding part gives ExpressGNN the extra representation power to perform accurate inference in the current knowledge graph.

\section{Experiments}\label{sec:exp}

\begin{wraptable}[9]{R}{0.36\textwidth}
\vspace{-14.8mm}
    \caption{\small Statistics of datasets.\label{tab:data_set}}
\vspace{-3mm}
\centering
\resizebox{0.35\textwidth}{!}{
    \label{tab:data_set}
    \begin{tabular}{ll@{\hspace{-3mm}}r@{\hspace{2pt}}r@{\hspace{2pt}}r}
    \toprule
    \multicolumn{2}{l}{\multirow{2}{*}{Dataset}}  &\#entity  & \#ground & \#ground \\
    \multicolumn{2}{l}{}                          &         & predicate & formula   \\ \hline
    \multicolumn{2}{l}{FB15K-237}                 & 15K     & 50M       & 679B      \\
    \multicolumn{2}{l}{Cora (avg)}                & 616     & 157K      & 457M     \\ \hline
    \multicolumn{1}{l|}{\multirow{5}{*}{\rotatebox{90}{ Kinship}}} & S1     & 62  & 50K       & 550K      \\
    \multicolumn{1}{l|}{}                         & S2       & 110  & 158K      & 3M        \\
    \multicolumn{1}{l|}{}                         & S3       & 160  & 333K      & 9M        \\
    \multicolumn{1}{l|}{}                         & S4       & 221  & 635K      & 23M       \\
    \multicolumn{1}{l|}{}                         & S5       & 266  & 920K      & 39M       \\ \hline
    \multicolumn{1}{l|}{\multirow{5}{*}{\rotatebox{90}{ UW-CSE}}} & AI  & 300     & 95K       & 73M       \\
    \multicolumn{1}{l|}{}                         & {\small Graphics} & 195 & 70K       & 64M       \\
    \multicolumn{1}{l|}{}                         & {\small Language} & 82  & 15K       & 9M        \\
    \multicolumn{1}{l|}{}                         & {\small Systems}  & 277 & 95K       & 121M      \\
    \multicolumn{1}{l|}{}                         & {\small Theory}   & 174 & 51K       & 54M       \\ 
    \bottomrule
    \end{tabular}
}
\end{wraptable}
Our experiments show that mean field approximation with ExpressGNN
enables efficient and effective probabilistic logic inference and lead to to state-of-the-art results in several benchmark and large datasets.

\textbf{Benchmark datasets.} (i) UW-CSE contains information of students and professors in five department (AI, Graphics, Language, System, Theory)~\citep{richardson2006markov}. (ii) Cora~\citep{singla2005discriminative} contains a collection of citations to computer science research papers. It is split into five subsets according to the research field. (iii) synthetic Kinship datasets contain kinship relationships (e.g., \texttt{Father}, \texttt{Brother}) and resemble the popular Kinship dataset~\citep{denham1973detection}. (iv) FB15K-237 is a large-scale knowledge base~\citep{toutanova2015observed}. Statistics of datasets are provided in Table~\ref{tab:data_set}. See more details of datasets in Appendix~\ref{app:exp-details}.

\subsection{Ablation study and comparison to strong MLN inference methods}
\label{sec:mln_infer}
\vspace{-1mm}
We conduct experiments on Kinship, UW-CSE and Cora, since other baselines can only scale up to these datasets. We use the original logic formulae provided in UW-CSE and Cora, and use hand-coded rules for Kinship. The weights for all formulae are set to 1. We use area under the precision-recall curve (AUC-PR) to evaluate deductive inference accuracy for {\bf predicates never seen during training}, and under the {\bf open-world} setting.\footnote{In Appendix~\ref{app:close_world_infer}, we report the performance under the closed-world setting as in the original works.} See Appendix~\ref{app:exp-details} for more details. Before comparing to other baselines, we first perform an ablation study for ExpressGNN in Cora to explore the trade-off between GNN and tunable embeddings.

\setlength\tabcolsep{3pt}
\begin{wraptable}[11]{R}{0.4\textwidth}
\centering
\vspace{-6.5mm}
\caption{\small AUC-PR for different combinations of GNN and tunable embeddings. Tune~$d$ stands for $d$-dim tunable embeddings and GNN~$d$ stands for $d$-dim GNN embeddings.}
\vspace{-2mm}
\resizebox{0.4\textwidth}{!}{
\label{tab:ablation}
\begin{tabular}{@{}lrrrrr@{}}
\toprule
\multirow{2}{*}{Model} & \multicolumn{5}{c}{Cora} \\ \cmidrule(l){2-6} 
 & \multicolumn{1}{c}{S1} & \multicolumn{1}{c}{S2} & \multicolumn{1}{c}{S3} & \multicolumn{1}{c}{S4} & \multicolumn{1}{c}{S5}\\ \midrule
Tune64 & 0.57 & 0.74 & 0.34 & 0.55 & 0.70 \\
GNN64 & 0.57 & 0.58 & 0.38 & 0.54 & 0.53  \\
GNN64+Tune4 & 0.61 & 0.75 & 0.39 & 0.54 & 0.70\\
\midrule
Tune128 & \textbf{0.62} & 0.76 & 0.42 & \textbf{0.60} & 0.73 \\
GNN128 & 0.60 & 0.59 & 0.45 & 0.55 & 0.61 \\
GNN64+Tune64 & \textbf{0.62} & \textbf{0.79} & \textbf{0.46} & 0.57 & \textbf{0.75}  \\ 
\bottomrule
\end{tabular}}
\end{wraptable}
\textbf{Ablation study.} The number of parameters in GNN is independent of entity size, but it is less expressive. The number of parameters in the tunable component is linear in entity size, but it is more expressive. Results on different combinations of these two components are shown in Table~\ref{tab:ablation}, which are consistent with our analytical result: GNN alone is not expressive enough. %

It is observed that GNN64+Tune4 has comparable performance with Tune64, but consistently better than GNN64. However, the number of parameters in GNN64+Tune4 is $O(64^2 +  4 |\Ccal|)$, while that in Tune64 is $O(64 |\Ccal|)$. A similar result is observed for GNN64+Tune64 and Tune128. Therefore, ExpressGNN as a combination of two types of embeddings can possess the advantages of both having a small number of parameters and being expressive. Therefore, we will use ExpressGNN throughout the rest of the experiments with hyperparameters optimized on the validation set. See Appendix~\ref{app:exp-details} for details.

\textbf{Inference accuracy.} We evaluate the inference accuracy of ExpressGNN against a number of state-of-the-art MLN inference algorithms: (i) MCMC (Gibbs Sampling)~\citep{gilks1995markov, richardson2006markov}; (ii) Belief Propagation (BP)~\citep{yedidia2001generalized}; (iii) Lifted Belief Propagation (Lifted BP)~\citep{singla2008lifted}; (iv) MC-SAT~\citep{poon2006sound}; (v) Hinge-Loss Markov Random Field (HL-MRF)~\citep{bach2015hinge}. 
Results are shown in Table~\ref{table:inference}.

\begin{table*}[h]
\centering
\caption{\small Inference accuracy (AUC-PR) of different methods on three benchmark datasets.}
\vspace{-1.5mm}
\label{table:inference}
\resizebox{0.85\textwidth}{!}{
\begin{tabular}{@{}lccccccccccc@{}}
\toprule
\multirow{2}{*}{Method} & \multicolumn{5}{c}{Kinship} & \multicolumn{5}{c}{UW-CSE} & \multicolumn{1}{c}{Cora} \\ \cmidrule(l){2-12} 
 & S1 & S2 & S3 & S4 & \multicolumn{1}{r|}{S5} & AI & Graphics & Language & Systems & \multicolumn{1}{r|}{Theory} & (avg) \\ 
 \midrule
MCMC & 0.53 & - & - & - & \multicolumn{1}{c|}{-} & - & - & - & - & \multicolumn{1}{c|}{-} & - \\
BP / Lifted BP & 0.53 & 0.58 & 0.55 & 0.55 & \multicolumn{1}{c|}{0.56} & 0.01 & 0.01 & 0.01 & 0.01 & \multicolumn{1}{c|}{0.01} & - \\
MC-SAT & 0.54 & 0.60 & 0.55 & 0.55 & \multicolumn{1}{c|}{-} & 0.03 & 0.05 & 0.06 & 0.02 & \multicolumn{1}{c|}{0.02} & - \\
HL-MRF & \textbf{1.00} & \textbf{1.00} & \textbf{1.00} & \textbf{1.00} & \multicolumn{1}{c|}{-} & 0.06 & 0.06 & 0.02 & 0.04 & \multicolumn{1}{c|}{0.03} & - \\
\midrule
ExpressGNN & 0.97 & 0.97 & 0.99 & 0.99 & \multicolumn{1}{c|}{0.99} & \textbf{0.09} & \textbf{0.19} & \textbf{0.14} & \textbf{0.06} & \multicolumn{1}{c|}{\textbf{0.09}} & \textbf{0.64} \\
\bottomrule
\end{tabular}
}
\vspace{-1mm}
\end{table*}

A hyphen in the entry indicates that the inference is either out of memory or exceeds the time limit (24 hours). Note that since the lifted BP is guaranteed to get identical results as BP~\citep{singla2008lifted}, the results of these two methods are merged into one row. For UW-CSE, the results suggest that ExpressGNN consistently outperforms all baselines. On synthetic Kinship, since the dataset is noise-free, HL-MRF achieves the score of 1 for the first four sets. ExpressGNN yields similar but not perfect scores for all the subsets, presumably caused by the stochastic nature of our sampling and optimization method.

\textbf{Inference efficiency.} The inference time on UW-CSE and Kinship are summarized in Figure~\ref{fig:time} (Cora is omitted as none of the baselines is feasible). As the size of the dataset grows linearly, inference time of all baseline methods grows exponentially. ExpressGNN maintains a nearly constant inference time with the increasing size of the dataset, demonstrating strong scalability. For HL-MRF, while maintaining a comparatively short wall-clock time, it exhibits an exponential increase in the space complexity. Slower methods such as MCMC and BP becomes infeasible for large datasets. ExpressGNN outperforms all baseline methods by at least one or two orders of magnitude. 
\begin{figure}[ht!]
    \centering
    \begin{tabular}[b]{cc}
    	\begin{subfigure}[]{0.52\textwidth}
        	\centering
            \resizebox{\textwidth}{!}{
            \begin{tabular}{@{}lccccc@{}}
            \toprule
            \multirow{2}{*}{Method} & \multicolumn{5}{c}{Inference Time (minutes)} \\ \cmidrule(l){2-6} 
            \multicolumn{1}{c}{} & AI & Graphics & Language & Systems & Theory \\ \midrule
            MCMC & $>$24h & $>$24h & $>$24h & $>$24h & $>$24h \\
            BP & 408 & 352 & 37 & 457 & 190 \\
            Lifted BP & 321 & 270 & 32 & 525 & 243 \\
            MC-SAT & 172 & 147 & 14 & 196 & 86 \\
            HL-MRF & 135 & 132 & 18 & 178 & 72 \\ \midrule
            ExpressGNN & \textbf{14} & \textbf{20} & \textbf{5} & \textbf{7} & \textbf{13} \\
            \bottomrule
            \end{tabular}
            }
    	\end{subfigure}
    	\hfill
        \begin{subfigure}[]{0.46\textwidth}
            \centering
        	\includegraphics[width=0.80\textwidth]{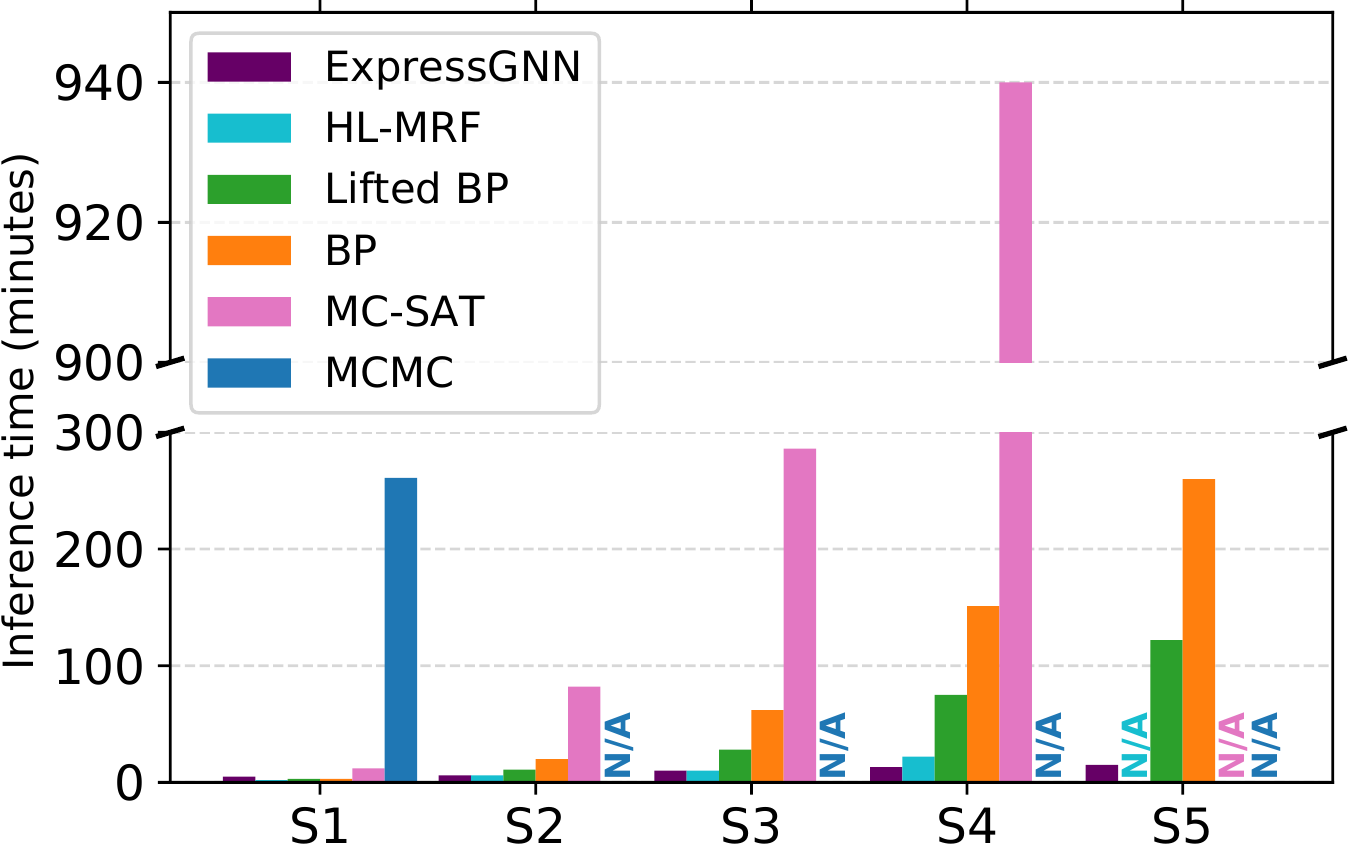}
        \end{subfigure}
    \end{tabular}
    \vspace{-3mm}
    \caption{\small {\it Left} / {\it Right}: Inference time on UW-CSE / Kinship respectively. N/A indicates the method is infeasible.}
    \label{fig:time}
\end{figure}

\vspace{-2mm}
\subsection{Large-scale knowledge base completion}
\vspace{-1mm}
We use a large-scale dataset, FB15K-237~\citep{toutanova2015observed}, to show the scalability of ExpressGNN. Since none of the aforementioned probabilistic inference methods are tractable on this dataset, we compare with several state-of-the-art supervised methods for knowledge base completion: (i) Neural Logic Programming (Neural LP) \citep{yang2017differentiable}; (ii) Neural Tensor Network (NTN)~\citep{socher2013reasoning}; (iii) TransE~\citep{bordes2013translating}. In these knowledge completion experiments, we follow the setting in~\citep{qu2019gmnn} to add a discriminative loss $\sum_{r(a_r)\in\mathcal{O}} \log Q(r(a_r))$ to better utilize observed data. We use Neural LP to generate candidate rules and pick up those with high confidence scores for ExpressGNN. See Appendix~\ref{app:rules} for examples of logic formulae used in experiments. For competitor methods, we use default tuned hyperparameters, which can reproduce the experimental results reported in their original works.

\textbf{Evaluation.} Given a query, e.g., $r(c,c')$, the task is to rank the query on top of all possible grounding of $r$. For evaluation, we compute the Mean Reciprocal Ranks (MRR), which is the average of the reciprocal rank of all the truth queries, and Hits@10, which is the percentage of truth queries that are ranked among top 10. Following the protocol proposed in~\citep{yang2017differentiable, bordes2013translating}, we also use \textit{filtered} rankings.

\begin{table}[h]
\vspace{-2mm}
    \centering
    \begin{minipage}{.65\linewidth}
        \centering
        \caption{Comparison on FB15K-237 with varied training set size.}
        \label{tab:scale}
        \resizebox{\textwidth}{!}{
        \begin{tabular}{@{}lcccccccccc@{}}
        \toprule
        {\multirow{2}{*}{Model}} & \multicolumn{5}{c}{MRR} &  \multicolumn{5}{c}{Hits@10} \\ \cmidrule(lr){2-6} \cmidrule(l){7-11}
        \multicolumn{1}{c}{} & \multicolumn{1}{c}{0\%} & \multicolumn{1}{c}{5\%} & \multicolumn{1}{c}{10\%} & \multicolumn{1}{c}{20\%} &
        \multicolumn{1}{c}{100\%} &  \multicolumn{1}{c}{0\%} & \multicolumn{1}{c}{5\%} & \multicolumn{1}{c}{10\%} &  \multicolumn{1}{c}{20\%} & \multicolumn{1}{c}{100\%} \\ \midrule
        Neural LP & 0.01 & 0.13 & 0.15 & 0.16 & 0.24 & 1.5 & 23.2 & 24.7 & 26.4 & 36.2 \\
        NTN & 0.09 & 0.10 & 0.10 & 0.11 & 0.13 & 17.9 & 19.3 & 19.1 & 19.6 & 23.9 \\
        TransE & 0.21 & 0.22 & 0.22 & 0.22 & 0.28 & 36.2 & 37.1 & 37.7 & 38.0 & 44.5 \\ \midrule
        ExpressGNN & \textbf{0.42} & \textbf{0.42} & \textbf{0.42} & \textbf{0.44} & \textbf{0.45} & \textbf{53.1} & \textbf{53.1} & \textbf{53.3} & \textbf{55.2} & \textbf{57.3} \\ \bottomrule
        \end{tabular}
        }
    \end{minipage}\hfill
    \begin{minipage}{.29\linewidth}
        \centering
        \caption{Inductive knowledge completion on FB15K-237.}
        \label{tab:transfer}
        \resizebox{\textwidth}{!}{
        \begin{tabular}{@{}lcc@{}}
        \toprule
        Model & \multicolumn{1}{c}{MRR} & \multicolumn{1}{c}{Hits@10} \\ \midrule
        Neural LP & 0.01 & 2.7 \\
        NTN & 0.00 & 0.0 \\
        TransE & 0.00 & 0.0 \\ \midrule
        ExpressGNN & \textbf{0.18} & \textbf{29.3} \\ \bottomrule
        \end{tabular}
        }
    \end{minipage}
\end{table}

\textbf{Data efficiency in transductive setting.} We demonstrate the data efficiency of using logic formula and compare ExpressGNN with aforementioned supervised approaches. More precisely, we follow \citep{yang2017differentiable} to split the knowledge base into facts / training / validation / testing sets, vary the size of the training set from 0\% to 100\%, and feed the varied training set with the same complete facts set to models for training. Evaluations on testing set are given in Table~\ref{tab:scale}. It shows with small training data ExpressGNN can generalize significantly better than supervised methods. With more supervision, supervised approaches start to close the gap with ExpressGNN.
This also suggests that high confidence logic rules indeed help us generalize better under small training data.

\textbf{Inductive ability.} To demonstrate the inductive learning ability of ExpressGNN, we conduct experiments on FB15K-237 where training and testing use disjoint sets of relations. To prepare data for such setting, we first randomly select a subset of relations, and restrict the test set to relations in this selected subset, which is similar to~\cite{bordes2013translating}.
Table~\ref{tab:transfer} shows the experimental results. As expected, in this inductive setting, supervised transductive learning methods such as NTN and TransE drop to zero in terms of MRR and Hits@10\footnote{The MRR and Hits@10 are both smaller than 0.01 for NTN and TransE.}. Neural LP performs inductive learning and generalizes well to new entities in the test set as discussed in \cite{yang2017differentiable}. However, in our inductive setting, where all the relations in the test set are new, Neural LP is not able to achieve good performance as reported in Table~\ref{tab:transfer}.
In contrast, ExpressGNN can directly exploit first-order logic and is much less affected by the new relations, and achieve reasonable performance at the same scale as the non-inductive setting.

\vspace{-1mm}
\section{Conclusion}
\vspace{-1mm}

Our analysis shows that GNN while being suitable for probabilistic logic inference in MLN, is not expressive enough. Motivated by this analysis, we propose ExpressGNN, an integrated GNN and tunable embedding approach, which has a trade-off between model size and expressiveness, and leads to scalable and effective logic inference in both transductive and inductive experiments. ExpressGNN opens up many possibilities for future research such as formula weight learning with variational inference, incorporating entity features and neural tensorized logic formulae, and addressing challenging datasets such as GQA~\citep{hudson2019gqa}.

\clearpage
\newpage

\clearpage
\newpage
\appendix

\noindent {\LARGE\bf Appendix}

\section{Proof of Theorems}
\label{app:thm-proof}
\thmAugKnowBase*

\begin{proof}For simplicity, we use $\gG$ and $\gG'$ to represent $\gG_{\gK}$ and $\gG_{\overline{\gK}}$ in this proof.
\begin{itemize}
    \item[2.] Let us first assume statement 1 is true and prove statement 2. 
    
    The neighbors of $H:=\sbr{r(c_1,\ldots,c_{n})=v}$ and $H':=\sbr{r(c_1',\ldots,c_{n}')=v}$ are 
    \begin{align}
        \gN\rbr{H}=\cbr{\rbr{c_i,i} : i =1,\ldots,n}~\text{and}~\gN\rbr{H'}=\cbr{\rbr{c_i',i} : i =1,\ldots,n}
    \end{align}
    where $i$ represents the edge type. It is easy to see that $H$ and $H'$ are indistinguishable in $\gG'$ if and only if $c_i'$ and $c_i$ are indistinguishable in $\gG'$ for $i=1,\ldots,n$. By statement 1, $c_i'$ and $c_i$ are indistinguishable in $\gG'$ if and only if $c_i'$ and $c_i$ are indistinguishable in $\gG$. Hence, statement 2 is true. Now it remains to prove statement 1.
    \item[1.] ($\Leftarrow$) If $c$ and $c'$ are distinguishable in $\gG$, it is easy to see $c$ and $c'$ are also distinguishable in the new graph $\gG'$. The reason is that the newly added nodes $\gH$ are of different types from the observed nodes $\gO$ in $\gG$, so that these newly added nodes can not make two distinguishable nodes to become indistinguishable. 
    
    ($\Rightarrow$) Assume that $c$ and $c'$ are indistinguishable in $\gG$, we will prove they are indistinguishable in $\gG'$ using MI (mathematical induction). The idea is to construct the new graph $\gG'$ by connecting the unobserved nodes in $\gH$ in a particular order. More specifically, we first connect all unobserved grounded predicates $\sbr{r(c_1,\ldots,c_{n})=?}\in \gH$ to their first arguments $c_1$, and the resulting graph is denoted by $\gG^{(1)}$. Then we can connect all $\sbr{r(c_1,c_2,\ldots,c_{n})=?}\in \gH$ to their second arguments $c_2$ and denote the resulting graph by $\gG^{(2)}$. In this way, we obtain a sequence of graphs $\cbr{\gG^{(k)}}_{k=1}^R$ where $R:=\max\cbr{n: r\in\gR}$ is the maximal number of arguments. It is clear that $\gG'= \gG^{(R)}$. In the following, we will use MI to prove that for all $k=1,\ldots,R$, if $c$ and $c'$ are indistinguishable in $\gG$, then they are indistinguishable in $\gG^{(k)}$.
    
    \underline{Proof of {\bf (MI 1)}}:
    
    Consider any predicate $r\in\gR$. For any two indistinguishable nodes $c,c'$ in $\gG$, $\#\cbr{ r(c,\ldots):~\text{observed} } = \#\cbr{ r(c',\ldots):~\text{observed} }$. Hence, it is obvious that
    \begin{align}
        \#\cbr{ r(c,\ldots):~\text{unobserved} } = \#\cbr{ r(c',\ldots):~\text{unobserved} }=M.
    \end{align}
    Before connected to the graph $\gG$, the unobserved nodes $\cbr{r(\cdot):~\text{unobserved}}$ are all indistinguishable because they are of the same node-type. Now we connect all these unobserved nodes to its first argument. Then $c$ is connected to $\cbr{ r(c,\ldots):~\text{unobserved} }$ and $c'$ is connected to $\cbr{ r(c',\ldots):~\text{unobserved} }$. Since both $c$ and $c'$ are connected $M$ unobserved nodes and these nodes are indistinguishable, $c$ and $c'$ remain to be indistinguishable. Also, after connected to its first argument, $r(c,\ldots)$ and $r(c',\ldots)$ are indistinguishable if and only if $c$ and $c'$ are indistinguishable, which is obvious. 
    
    Similarly, we can connect all unobserved grounded predicates to its first argument. In the resulting graph, two nodes are indistinguishable if they are indistinguishable in $\gG$.
    
    \underline{Assumption {\bf (MI $k$)}}:
    
    Assume that after connecting all unobserved grounded predicates to their first $k$ arguments, the constant nodes in the resulting graph, $\gG^{(k)}$, are indistinguishable if they are indistinguishable in $\gG$. 
    
    \underline{Proof of {\bf (MI $k+1$)}}: 
    
    The constants $\gC$ in $\gG$ can be partitioned into $N$ groups $\gC = \bigcup_{i=1}^N\gC^{(i)}$, where the constants in the same group $\gC^{(i)}$ are indistinguishable in $\gG$ (and also indistinguishable in $\gG^{(k)}$).

    Consider a predicate $r^*\in\gR$. The set of unobserved grounded predicates where the $(k+1)$-th argument is $c$ can be written as
    \begin{align}
       & \cbr{r^*(\ldots,c_{k+1}=c,\ldots):~\text{unobserved} }\\
       =&\bigcup_{i_1=1}^N \cdots \bigcup_{i_k=1}^N \cbr{r^*(c_1,\ldots,c_k,c_{k+1}=c,\ldots): c_i\in\gC^{(i_1)},\ldots,c_k\in\gC^{(i_k)},~\text{unobserved}}.
    \end{align}
    Similar to the arguments in {\bf (MI 1)}, for any two indistinguishable nodes $c$ and $c'$, for any fixed sequence of groups $i_1,\ldots,i_k$, the size of the following two sets are the same:
     \begin{align}
       M(i_1,\ldots,i_k)= &\Scale[0.9]{\#\cbr{r^*(c_1,\ldots,c_k,c_{k+1}=c,\ldots): c_i\in\gC^{(i_1)},\ldots,c_k\in\gC^{(i_k)},~\text{unobserved}}}\\
        =&\Scale[0.9]{\#\cbr{r^*(c_1,\ldots,c_k,c_{k+1}=c',\ldots): c_i\in\gC^{(i_1)},\ldots,c_k\in\gC^{(i_k)},~\text{unobserved}}}.
    \end{align}
    Also, all grounded predicates in the above two sets are indistinguishable in $\gG^{(k)}$ because their first $k$ arguments are indistinguishable. Hence, these are two sets of $M(i_1,\ldots,i_k)$ many indistinguishable nodes. In conclusion, the two sets $\cbr{r^*(\ldots,c_{k+1}=c,\ldots):~\text{unobserved} }$ and $\cbr{r^*(\ldots,c_{k+1}=c',\ldots):~\text{unobserved} }$ are indistinguishable in $\gG^{(k)}$ if $c$ and $c'$ are indistinguishable. 
    
    Now we connect all unobserved $r^*(\cdot)$ to their $(k+1)$-th arguments. Then the constant node $c$ is connected to $\cbr{r^*(\ldots,c_{k+1}=c,\ldots):~\text{unobserved} }$ and $c'$ is connected to $\cbr{r^*(\ldots,c_{k+1}=c',\ldots):~\text{unobserved} }$. Since these two sets are indistinguishable, then $c$ and $c'$ remain to be indistinguishable. 
    
    Similarly, for other predicates $r\in\gR$, we can connect all unobserved grounded predicates $r(\cdot)$ to their $(k+1)$-th arguments. In the resulting graph, $\gG^{(k+1)}$, any pair of two nodes will remain indistinguishable if they are indistinguishable in $\gG$.
    
\end{itemize}

\end{proof}

\thmautomorphism*
\begin{proof} A graph isomorphism from $G$ to itself is called automorphism, so in this proof, we will use the terminology - automorphism - to indicate such a self-bijection. 

$(\Longleftarrow)$ We first prove the sufficient condition:
\begin{quote}
    If $\exists$ automorphism $\pi$ on the graph $\gG_{\gK}$ such that $\pi(c_i)=c_i',\forall i =1,...,n$, then for any $r\in\gR$, $r(c_1,\ldots,c_{n})$ and $r(c_1',\ldots,c_{n}')$ have the same posterior in any MLN.
\end{quote}
MLN is a graphical model that can also be represented by a factor graph $\text{MLN}= (\gO\cup \gH,\gF_g, \gE) $ where grounded predicates (random variables) and grounded formulae (potential) are connected. We will show that $\exists$ an automorphism $\phi$ on MLN such that $\phi\rbr{r(c_1,\ldots,c_{n})}=r(c_1',\ldots,c_{n}')$. Then the sufficient condition is true. This automorphism $\phi$ is easy to construct using the automorphism $\pi$ on $\gG_{\gK}$. More precisely, we define $\phi:(\gO\cup \gH,\gF_g)\rightarrow (\gO\cup \gH,\gF_g)$ as
\begin{align}
    \phi(r(a_r)) = r(\pi(a_r)),~~\phi(f(a_f)) = f(\pi(a_f)),
\end{align}
for any predicate $r\in\gR$, any assignments $a_r$ to its arguments, any formula $f\in \gF$, and any assignments $a_f$ to its arguments. It is easy to see $\phi$ is an automorphism:
\begin{enumerate}[leftmargin=*,nolistsep]
    \item Since $\pi$ is a bijection, apparently $\phi$ is also a bijection. 
    \item The above definition preserves the biding of the arguments. $r(a_r)$ and $f(a_f)$ are connected if and only if $\phi(r(a_r))$ and $f(\pi(a_f))$ are connected.
    \item Given the definition of $\pi$, we know that $r(a_r)$ and $r(\pi(a_r))$ have the same observation value. Therefore, in MLN, $\texttt{NodeType}(r(a_r))=\texttt{NodeType}(\phi(r(a_r)))$.
\end{enumerate}
This completes the proof of the sufficient condition.

$(\Longrightarrow)$ To prove the necessary condition, it is equivalent to show given the following assumption
    \begin{quote}
       {\bf (A 1)}: there is no automorphism $\pi$ on the graph $\gG_{\gK}$ such that $\pi(c_i)=c_i',\forall i =1,...,n$,
   \end{quote}
the following statement is true:
   \begin{quote}
       there must exists a MLN and a predicate $r$ in it such that $r(c_1,\ldots,c_{n})$ and $r(c_1',\ldots,c_{n}')$ have different posterior.
   \end{quote}
   Before showing this, let us first introduce the {\bf factor graph representation of a single logic formula} $f$.

\fbox{
  \parbox{0.9\columnwidth}{
\begin{multicols}{2}
 A logic formula $f$ can be represented as a factor graph, $\Gcal_f = (\Ccal_f, \Rcal_f, \Ecal_f)$, where nodes on one side of the graph is the set of distinct constants $\Ccal_f$ needed in the formula, while nodes on the other side is the set of predicates $\Rcal_f$ used to define the formula. The set of edges, $\Ecal_f$, will connect constants to predicates or predicate negation. That is, an edge 
\begin{quote}
    $e=(c,r,i)$ between node $c$ and predicate $r$ exists, if the predicate $r$ use constant $c$ in its $i$-th argument.   
\end{quote}
We note that the set of distinctive constants used in the definition of logic formula are templates where actual constant can be instantiated from $\Ccal$. An illustration of logic formula factor graph can be found in Figure~\ref{fig:formula_graph}.
Similar to the factor graph for the knowledge base, we also differentiate the type of edges by the position of the argument.
\begin{Figure}
    \includegraphics[width=\linewidth]{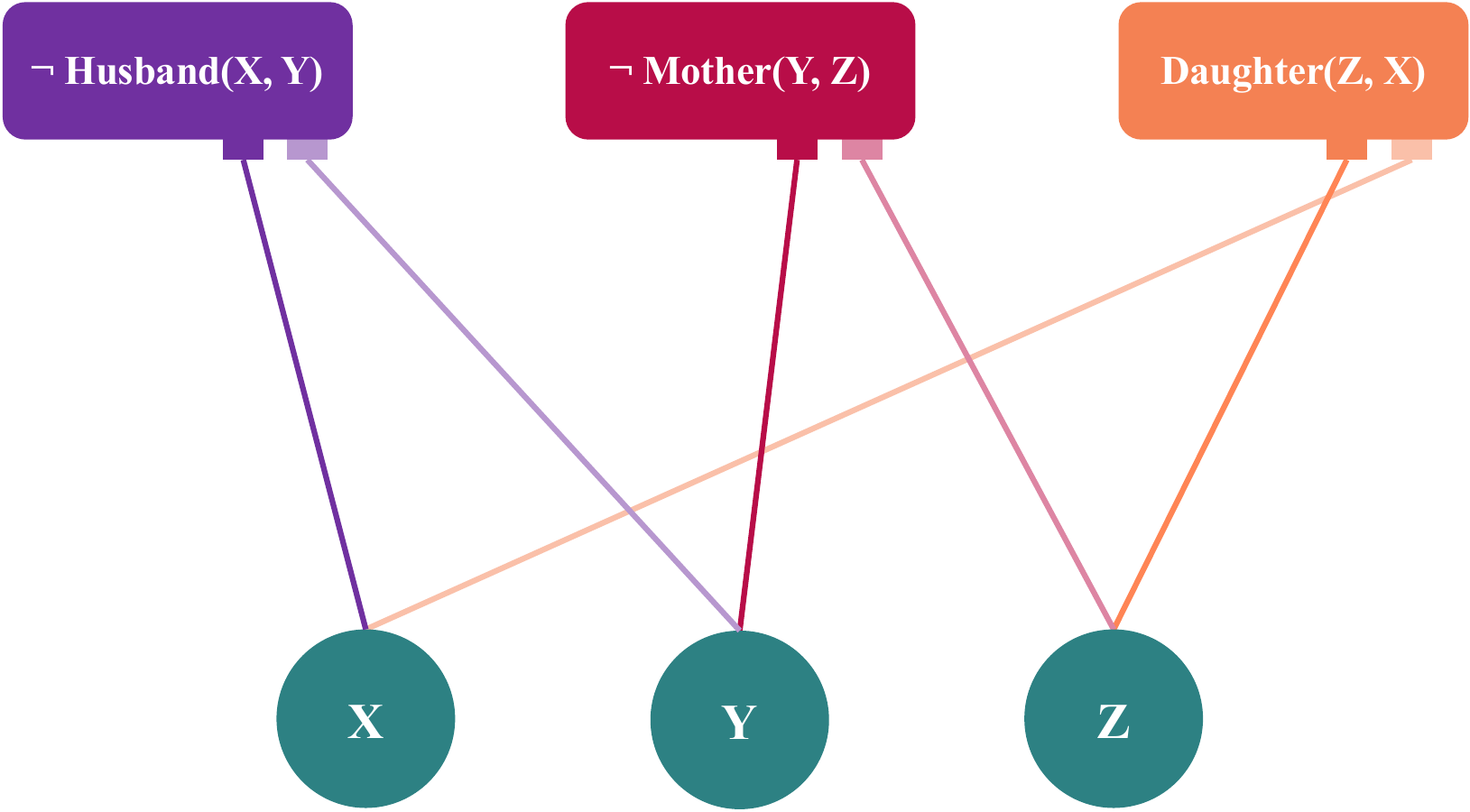}
    \captionof{figure}{An example of factor graph for the logic formula $\lnot \texttt{Husband(X,Y)} \lor \lnot \texttt{Mother(Y,Z)} \lor \texttt{Daughter(Z,X)}.$ }
    \label{fig:formula_graph}
\end{Figure}
\end{multicols}
}}

   Therefore, every single formula can be represented by a factor graph. We will construct a factor graph representation to define a particular formula, and show that the MLN induced by this formula will result in different posteriors for $r(c_1,\ldots,c_{n})$ and $r(c_1',\ldots,c_{n}')$. The factor graph for the formula is constructed in the following way (See Figure~\ref{fig:app-counter2-2} as an example of the resulting formula constructed using the following steps):

\begin{enumerate}
  
    \item[(i)]Given the above assumption {\bf (A 1)}, we \underline{claim} that:
        
        $\exists$ a subgraph $\gG^*_{c_{1:n}}=\rbr{\gC^*_{c}, \gO^*_{c},\gE^*_{c}}\subseteq\gG_{\gK}$ such that all subgraphs $\gG_{c_{1:n}'}=\rbr{\gC_{c'}, \gO_{c'},\gE_{c'}}\subseteq \gG_{\gK}$ satisfy:
        \begin{quote}
           {\bf (Condition)} if there exists an isomorphism $\phi:\gG^*_{c_{1:n}}\rightarrow \gG_{c_{1:n}'}$ satisfying $\phi(c_i)=c_i',\forall i=1,\ldots,n$ after the observation values are IGNORED (that is, $\sbr{r_j(\cdots)=0}$ and $\sbr{r_j(\cdots)=1}$ are treated as the SAME type of nodes), then the set of fact nodes (observations) in these two graphs are different (that is, $\gO^*_{c} \neq \gO_{c'}$).
        \end{quote}
        The proof of this \underline{claim} is given at the end of this proof.
        \item[(ii)] Next, we use $\gG^*_{c_{1:n}}$ to define a formula $f$. We first initialize the definition of the formula value as 
        \begin{align}\label{eq:construct-f}
            f(c_1,\ldots,c_n,\tilde{c}_1,\ldots,\tilde{c}_n) = \rbr{ \land \cbr{\tilde{r}(a_{\tilde{r}}): \tilde{r}(a_{\tilde{r}})\in  \gG^*_{c_{1:n}}}}\Rightarrow r(c_1,\ldots,c_n).
        \end{align}
        Then, we change $\tilde{r}(a_{\tilde{r}})$ in this formula to the negation $\lnot\tilde{r}(a_{\tilde{r}})$ if the observed value of $\tilde{r}(a_{\tilde{r}})$ is 0 in $\gG^*_{c_{1:n}}$.
    \end{enumerate}
    We have defined a formula $f$ using the above two steps. Suppose the MLN only contains this formula $f$. Then 
    \begin{quote}
        the two nodes $r(c_1,\ldots,c_{n})$ and $r(c_1',\ldots,c_{n}')$ in this MLN must be distinguishable.
    \end{quote}
     The reason is, in MLN, $r(c_1,\ldots,c_{n})$ is connected to a grounded formula $f(c_1,\ldots,c_{n},\tilde{c}_1,\ldots,\tilde{c}_n)$, whose factor graph representation is $\gG_{c_{1:n}}^*\cup r(c_1,\ldots,c_{n})$.  In this formula, all variables are observed in the knowledge base $\gK$ except for $r(c_1,\ldots,c_{n})$ and and the observation set is $\gO_c^*$. The formula value is 
     \begin{align}\label{eq:ground-f}
          f(c_1,\ldots,c_{n},\tilde{c}_1,\ldots,\tilde{c}_n)=\rbr{ 1\Rightarrow r(c_1,\ldots,c_n)}.
     \end{align}
     Clarification: \Eqref{eq:construct-f} is used to {\bf define} a formula and $c_i$ in this equation can be replaced by other constants, while \Eqref{eq:ground-f} represents a {\bf grounded} formula whose arguments are exactly  $c_1,\ldots,c_{n},\tilde{c}_1,\ldots,\tilde{c}_n$.
     Based on {\bf (Condition)}, there is NO formula  $f(c_1',\ldots,c_{n}',\tilde{c}_1',\ldots,\tilde{c}_n')$ that contains $r(c_1',\ldots,c_{n}')$ has an observation set the same as $\gO_c^*$. Therefore, $r(c_1,\ldots,c_{n})$ and $r(c_1',\ldots,c_{n}')$ are distinguishable in this MLN.
    
\underline{Proof of claim}:

We show the existence by constructing the subgraph $\gG_{c_{1:n}}^*\subseteq \gG_{\gK}$ in the following way:
    \begin{enumerate}[wide]
    \item[(i)] First, we initialize the subgraph as $\gG_{c_{1:n}}^*:=\gG_{\gK}$. 
    Given assumption {\bf (A 1)} stated above, it is clear that 
    \begin{quote}
        {\bf (S 1)} $\forall$ subgraph $\gG'\subseteq \gG_{\gK}$, there is no isomorphism $\pi:\gG_{c_{1:n}}^*\rightarrow \gG'$ satisfying $\pi(c_i)=c_i',\forall i =1,\ldots,n$.
    \end{quote}
    \item[(ii)] Second, we need to check wether the following case occurs: 
    \begin{quote}
        {\bf (C 1)} $\exists$ a subgraph $\gG'=(\gC',\gO',\gE')$ such that (1) there EXISTS an isomorphism $\phi:\gG_{c_{1:n}}^*\rightarrow \gG'$ satisfying $\phi(c_i)=c_i',\forall i=1,\ldots,n$ after the observation values are IGNORED (that is, $\sbr{r_j(\cdots)=0}$ and $\sbr{r_j(\cdots)=1}$ are treated as the same type of nodes); and (2) the set of factor nodes (observations) in these two graphs are the same (that is, $\gO^*_c = \gO'$).
    \end{quote}
   
   \item[(iii)] Third, we need to modify the subgraph if the case {\bf (C 1)} is observed.
    Since $\left|\gG^*_{c_{1:n}} \right|\geq \left|\gG' \right|$, the only subgraph that will lead to the case ${\bf (C 1)}$ is the maximal subgraph $\gG^*_{c_{1:n}}$. The isomorphism $\phi$ is defined by ignoring the observation values, while the isomorphism $\pi$ in {\bf (S 1)} is not ignoring them. Thus, 
   
   \begin{quote}
        {\bf (S 1)} and {\bf (C 1)} $\Longrightarrow$  $\exists$ a set of nodes $S:=\cbr{\sbr{r_j(a^{(1)})=0},\ldots, \sbr{r_j(a^{(n)})=0}}$ such that 
   for any isomorphism $\phi$ satisfying the conditions in {\bf (C 1)}, the range $\phi(S)$ contains at least one node $\sbr{r_j(\cdot)=1}$ which has observation value 1.
   \end{quote}
   Otherwise, it is easy to see a contradiction to statement {\bf (S 1)}.
   \begin{quote}
            {\bf (M 1)} Modify the subgraph by $\gG^*_{c_{1:n}} \longleftarrow \gG^*_{c_{1:n}} \setminus S $. The nodes (and also their edges) in the set $S:=\cbr{\sbr{r_j(a^{(1)})=0},\ldots, \sbr{r_j(a^{(n)})=0}}$ are removed. 
    \end{quote}
    For the new subgraph $\gG^*_{c_{1:n}} $ after the modification {\bf (M 1)}, the case {\bf (C 1)} will not occur. Thus, we've obtained a subgraph that satisfies the conditions stated in the claim. Finally, we can remove the nodes that are not connected with $\cbr{c_1,\ldots,c_{n}}$ (that is, there is no path between this node and any one of $\cbr{c_1,\ldots,c_{n}}$). The remaining graph is connected to $\cbr{c_1,\ldots,c_{n}}$ and still satisfies the conditions that we need.
    \end{enumerate}
    
\end{proof}
\newpage
\section{Counter Examples}
\label{app:counter-examples}
We provide more examples in this section to show that it is more than a rare case that GNN embeddings alone are not expressive enough.

\subsection{Example 1}
\begin{figure}[h!]
    \centering
    \includegraphics[width=0.5\textwidth]{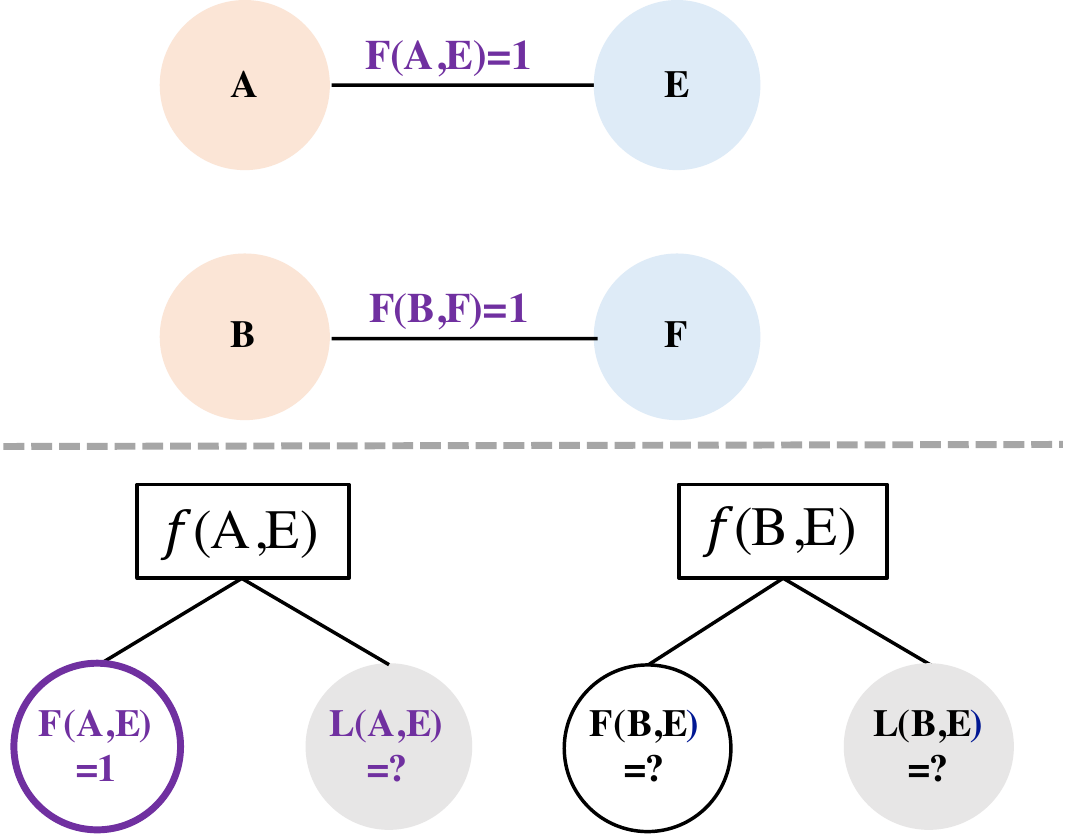}
    \caption{Example 1. {\it Top}: Knowledge base. {\it Bottom}: MLN}
    \label{fig:app-counter1}
\end{figure}

Unlike the example shown in main text, where \texttt{A} and \texttt{B} have OPPOSITE relation with \texttt{E}, Figure~\ref{fig:app-counter1} shows a very simple example where \texttt{A} and \texttt{B} have exactly the same structure which makes \texttt{A} and \texttt{B} indistinguishable and isomorphic. However, since (\texttt{A},\texttt{E}) and (\texttt{B},\texttt{E}) are not isomorphic, it can be easily seen that $\texttt{L}(\texttt{A},\texttt{E})$ has different posterior from $\texttt{L}(\texttt{B},\texttt{E})$.

\subsection{Example 2}
\begin{figure}[h!]
    \centering
    \includegraphics[width=0.5\textwidth]{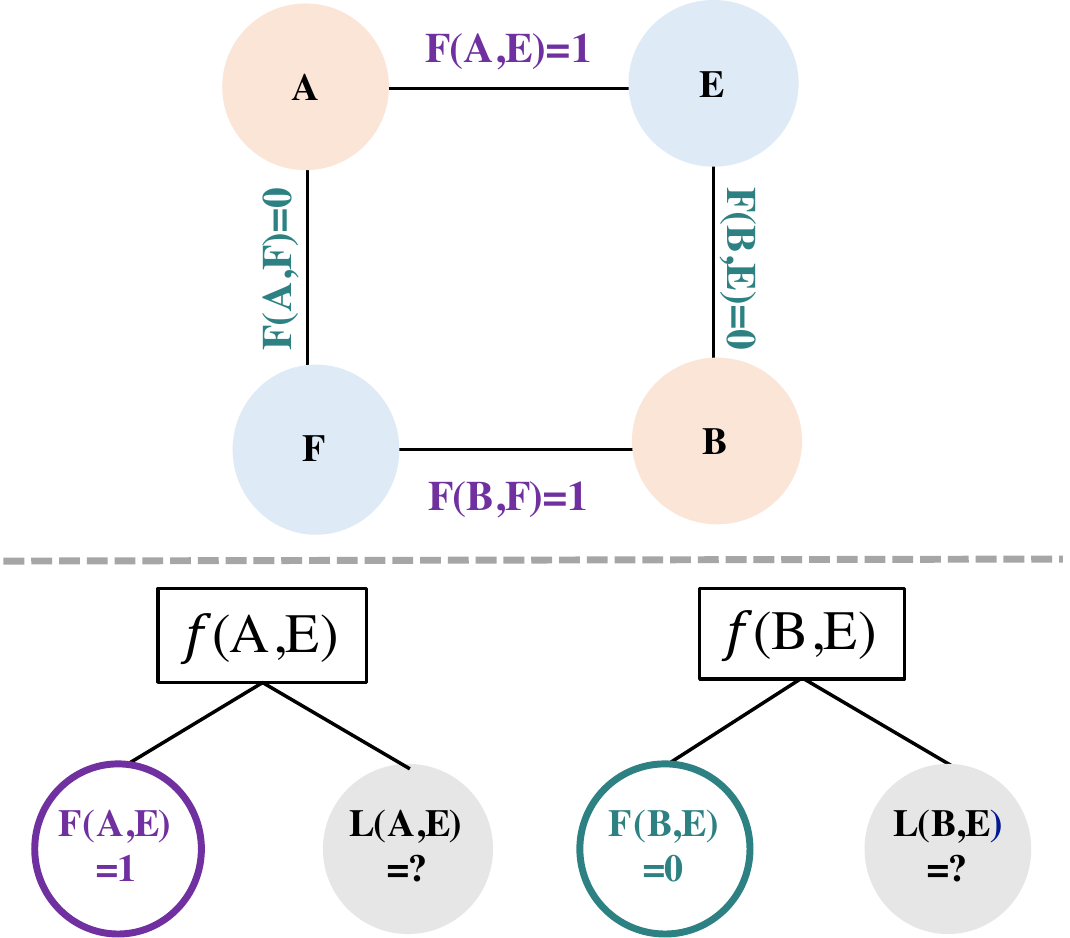}
    \caption{The same example as in Figure~\ref{fig:loopy-knowledge-base}. {\it Top}: Knowledge base. {\it Bottom}: MLN}
    \label{fig:app-counter2-1}
\end{figure}
Figure~\ref{fig:app-counter2-1} shows an example which is the same as in Figure~\ref{fig:loopy-knowledge-base}. However, in this example, it is already revealed in the knowledge base that $(\texttt{A},\texttt{E})$ and $(\texttt{B},\texttt{E})$ have different local structures as they are connected by different observations. That is,  $\rbr{\texttt{A},\sbr{\texttt{F}(\texttt{A},\texttt{E})=1}, \texttt{E}}$ and $\rbr{\texttt{B},\sbr{\texttt{F}(\texttt{B},\texttt{E})=0}, \texttt{E}}$ can be distinguished by GNN.

Now, we use another example in Figure~\ref{fig:app-counter2-2} to show that even when the local structures are the same, the posteriors can still be different, which is caused by the formulae.

\newpage
\begin{figure}[t!]
    \centering
    \includegraphics[width=\textwidth]{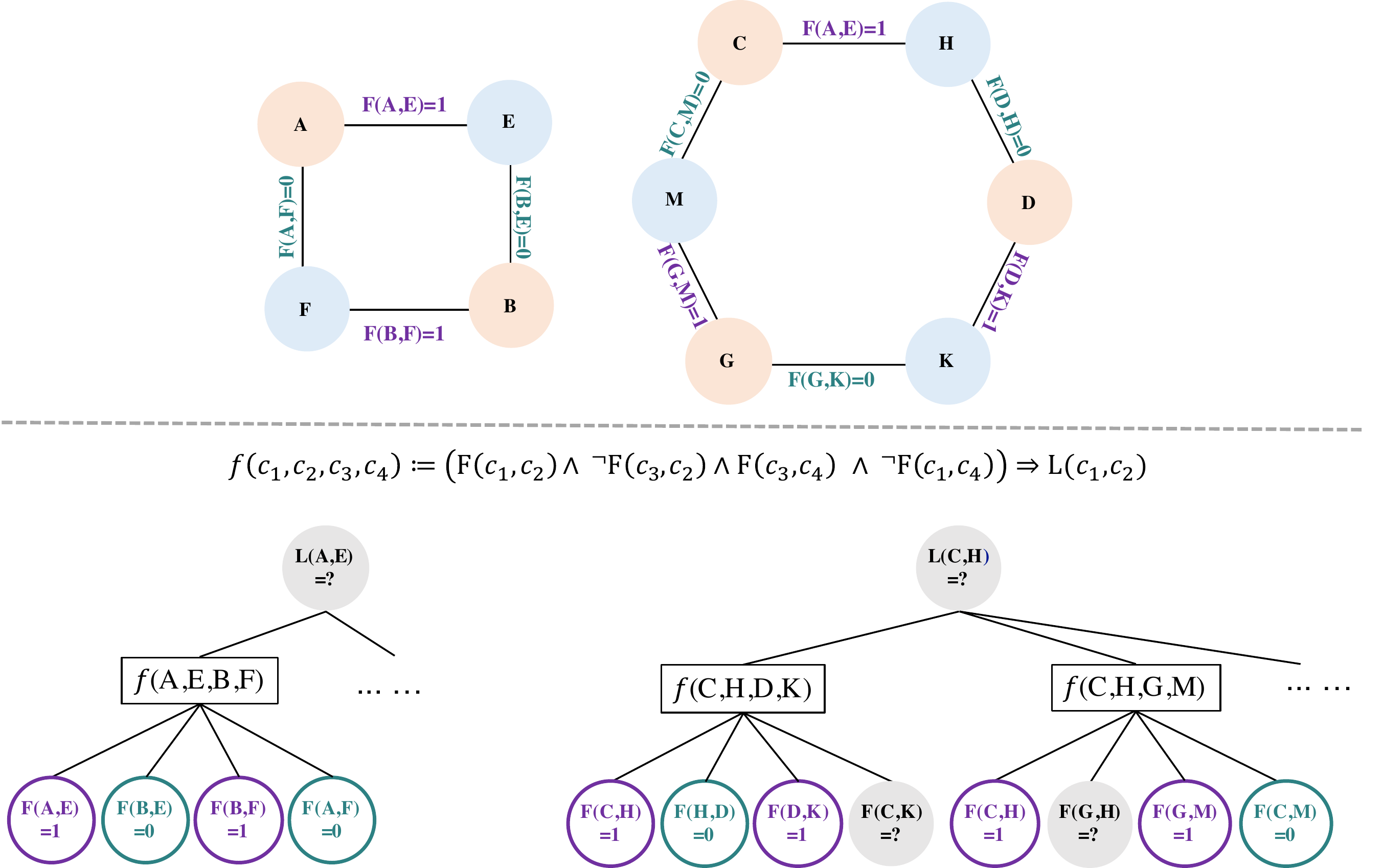}
    \caption{Example 2. {\it Top}: Knowledge base. {\it Bottom}: MLN}
    \label{fig:app-counter2-2}
\end{figure}

In Figure~\ref{fig:app-counter2-2}, $(\texttt{A},\texttt{E})$ and $(\texttt{C},\texttt{H})$ have the same local structure, so that the tuple $\rbr{\texttt{A},\sbr{\texttt{F}(\texttt{A},\texttt{E})=1}, \texttt{E}}$ and $\rbr{\texttt{C},\sbr{\texttt{F}(\texttt{C},\texttt{H})=1}, \texttt{H}}$ can NOT be distingushed by GNN. However, we can make use of subgraph $(\texttt{A}, \texttt{E}, \texttt{B}, \texttt{F})$ to define a formula, and then the resulting MLN gives different posterior to $\texttt{L}(\texttt{A},\texttt{E})$ and $\texttt{L}(\texttt{C},\texttt{H})$, as can be seen from the figure. Note that this construction of MLN is the same as the construction steps stated in the proof in Section~\ref{app:thm-proof}.

\newpage
\section{Experiment Settings}\label{app:exp-details}

\textbf{Experimental setup.}
All the experiments are conducted on a GPU-enabled (Nvidia RTX 2080 Ti) Linux machine powered by Intel Xeon Silver 4116 processors at 2.10GHz with 256GB RAM. We implement ExpressGNN using PyTorch and train it with Adam optimizer~\citep{kingma2014adam}. To ensure a fair comparison, we allocate the same computational resources (CPU, GPU and memory) for all the experiments. We use default tuned hyperparameters for competitor methods, which can reproduce the experimental results reported in their original works. For ExpressGNN, we use 0.0005 as the initial learning rate, and decay the learning rate by half for every 10 epochs without improvement in terms of validation loss. For Kinship, UW-CSE and Cora, we run ExpressGNN with a fixed number of iterations, and use the smallest subset from the original split for hyperparameter tuning. For FB15K-237, we use the original validation set to tune the hyperparameters.

Here are more details of the setup of ExpressGNN. We use a two-layer MLP with ReLU activation function as the nonlinear transformation for each embedding update step in the graph neural network in Algorithm~\ref{algo:gnn}. For different steps, we learn different MLP parameters. To increase the model capacity of ExpressGNN, we also use different MLP parameters for different edge type, and for a different direction of embedding aggregation. As we discussed in Section~\ref{sec:mln_infer}, the number of trainable parameters in GNN alone is independent of the number of entities in the knowledge base. Therefore, it has a minor impact on the computational cost by using different MLP parameters as described above. For each dataset, we search the configuration of ExpressGNN on either the validation set or the smallest subset. The configuration we search includes the embedding size, the split point of tunable embeddings and GNN embeddings, the number of embedding update steps, and the sampling batch size.

\textbf{Task and evaluation metrics.} 
The deductive logic inference task is to answer queries that typically involve single predicate. For example in UW-CSE, the task is to predict the \texttt{AdvisedBy}($c$,$c'$) relation for all persons in the set. In Cora, the task is to de-duplicate entities, and one of the query predicates is \texttt{SameAuthor}($c$,$c'$). As for Kinship, the task is to predict whether a person is male or female, i.e., \texttt{Male($c$)}. For each possible substitution of the query predicate with different entities, the model is tasked to predict whether it's true or not. Then we use the area under the precision-recall curve (AUC-PR) as the evaluation metric for inference accuracy. Due to the severe imbalance of positive and negative samples in typical logic reasoning tasks, the AUC-PR is better than the AUC-ROC to reflect the actual model performance and is widely used in the literature~\citep{richardson2006markov}.

For knowledge base completion, all methods are evaluated on test sets with Mean Reciprocal Rank (MRR) and Hits@10. Both are commonly used metrics for knowledge base completion. For each test query $r(c,c')$ with respect to relation $r$, the model is tasked to generate a rank list over all possible instantiations of $r$ and sort them according to the model's confidence on how likely this instantiation is true. Then MRR is computed as the average over all the reciprocal ranks of each $r(c,c')$ in its corresponding rank list, and Hits@10 is computed as the average times of ranking of the true fact in top 10 predictions. If two candidate entities have the same score, we break the tie by ranking the wrong ones ahead. This ensures a fair comparison for all methods. Additionally, before evaluation, the rank list will be filtered~\citep{yang2017differentiable, bordes2013translating} so that it does not contain any true fact other than $r(c,c')$ itself. 

\textbf{Sampling method.} As there are exponential many formulae in the sampling space, one cannot sample by explicitly enumerating all the formulae and permute them. On the other hand, not all ground formulae will contribute to the optimization during training. For example, a ground formula that only contains observed variables will not contribute gradients, as evaluating this formula is independent of the latent variable posterior. 

To overcome these challenges, we propose the following efficient sampling scheme: 1) to sample a ground formula we start from uniformly sampling a formula $f$ from the space of $\Fcal$; 2) shuffle its predicate space $\Rcal_f$ into a sequence; 3) for each predicate $r$ popped from the top of the $\Rcal_f$, with a probability of $p_{obs}$ we instantiate it as an observed variable and with a probability of $1 - p_{obs}$ it will become a uniformly sampled variable; 4) to instantiate an observed variable, we list all facts stored in the knowledge base with respect to predicate $r$ and uniformly sample from it. In the case where no fact can be found, or if we hit probability $1 - p_{obs}$, then the predicate $r$ will be instantiated with a random constant. Once a formula is fully instantiated, we examine its form and reject those without any latent variable. 

In experiments, we set $p_{obs}$ as 0.9. The intuition is that one wants to prioritize on sampling formulae containing both observed and latent variables. Otherwise, in the cases where a formula is fully latent, GNN is essentially optimizing towards learning the prior distribution determined by the form of formula and its weight, which is unlikely to be close to the actual posterior distribution. Additionally, for knowledge completion on FB15K-237, we further control the sample space to be query-related. Each time the model is fed with a query $r(c,c')$, we sample only the ground formulae with $r$ as the positive literal and containing constants $c$ and $c'$.

\section{Dataset Details}
\label{app:data-detail}

For experiments, we use four benchmark datasets: 1) The social network dataset UW-CSE~\citep{richardson2006markov} contains publicly available information of students and professors in the CSE department of UW. The dataset is split into five sets according to the home department of the entities. 2) The entity resolution dataset Cora~\citep{singla2005discriminative} consists of a collection of citations to computer science research papers. The dataset is also split into five subsets according to the field of research. 3) We introduce a synthetic dataset that resembles the popular Kinship dataset~\citep{denham1973detection}. The original dataset contains kinship relationships (e.g., \texttt{Father}, \texttt{Brother}) among family members in the Alyawarra tribe from Central Australia. We generate five sets by linearly increasing the number of entities. 4) The knowledge base completion benchmark FB15K-237~\citep{toutanova2015observed} is a generic knowledge base constructed from Freebase, which is designed to a more challenging variant of FB15K. More specifically, FB15K-237 is constructed by removing near-duplicate and inverse relations from FB15K. The dataset is split into training / validation / testing and we use the same split of facts from training as in prior work~\citep{yang2017differentiable}.

\subsection{Datasets statistics}
\label{app:data-statistics}
The complete statistics of the benchmark and synthetic datasets are shown in Table~\ref{tab:data_more_stats}. The statistics of Cora is averaged over its five splits. Examples of logic formulae used in four benchmark datasets are listed in Table~\ref{tab:all_rule}.

\begin{table}[t]
\caption{Complete statistics of the benchmark and synthetic datasets.}
\label{tab:data_more_stats}
\centering
\begin{tabular}{@{}lrrrrrr@{}}
\toprule
\multirow{2}{*}{\textbf{Dataset}} & \multirow{2}{*}{\# entity} & \multirow{2}{*}{\# relation} & \multirow{2}{*}{\# fact} & \multirow{2}{*}{\# query} & \# ground & \# ground \\
 &  &  &  &  & predicate & formula \\ \midrule
FB15K-237 & 15K & 237 & 272K & 20K & 50M & 679B \\ \midrule
Kinship-S1 & 62 & 15 & 187 & 38 & 50K & 550K \\
Kinship-S2 & 110 & 15 & 307 & 62 & 158K & 3M \\
Kinship-S3 & 160 & 15 & 482 & 102 & 333K & 9M \\
Kinship-S4 & 221 & 15 & 723 & 150 & 635K & 23M \\
Kinship-S5 & 266 & 15 & 885 & 183 & 920K & 39M \\ \midrule
UW-CSE-AI & 300 & 22 & 731 & 4K & 95K & 73M \\
UW-CSE-Graphics & 195 & 22 & 449 & 4K & 70K & 64M \\
UW-CSE-Language & 82 & 22 & 182 & 1K & 15K & 9M \\
UW-CSE-Systems & 277 & 22 & 733 & 5K & 95K & 121M \\
UW-CSE-Theory & 174 & 22 & 465 & 2K & 51K & 54M \\ \midrule
Cora-S1 & 670 & 10 & 11K & 2K & 175K & 621B \\
Cora-S2 & 602 & 10 & 9K & 2K & 156K & 431B \\
Cora-S3 & 607 & 10 & 18K & 3K & 156K & 438B \\
Cora-S4 & 600 & 10 & 12K & 2K & 160K & 435B \\
Cora-S5 & 600 & 10 & 11K & 2K & 140K & 339B \\
\bottomrule
\end{tabular}
\end{table}

\subsection{Synthetic Kinship Dataset}
\label{app:synthetic}
The synthetic dataset closely resembles the original Kinship dataset but with a controllable number of entities. To generate a dataset with $n$ entities, we randomly split $n$ entities into two groups which represent the first and second generation respectively. Within each group, entities are grouped into a few sub-groups representing the sister- and brother-hood. Finally, entities from different sub-groups in the first generation are randomly coupled and a sub-group in the second generation is assigned to them as their children. To generate the knowledge base, one traverse this family tree, and record all kinship relations for each entity. In this experiment, we generate five datasets by linearly increasing the number of entities. Examples of the first-order logic formulae used in the Kinship dataset is summarized in Table~\ref{tab:all_rule}.

\section{Inference with Closed-World Semantics for Baseline Methods}
\label{app:close_world_infer}

\setlength\tabcolsep{3pt}
\begin{table}[t]
\centering
\caption{Inference performance of competitors and our method under the closed-world semantics.}
\label{table:close_table}
\begin{tabular}{@{}lcccccccccc@{}}
\toprule
{\multirow{2}{*}{Method}} & \multicolumn{5}{c}{Cora} & \multicolumn{5}{c}{UW-CSE} \\ \cmidrule(l){2-11}
 & S1 & S2 & S3 & S4 & \multicolumn{1}{c|}{S5} & AI & Graphics & Language & Systems & Theory \\ \midrule
MCMC & 0.43 & 0.63 & 0.24 & 0.46 & \multicolumn{1}{c|}{0.56} & 0.19 & 0.04 & 0.03 & 0.15 & 0.08 \\
BP / Lifted BP & 0.44 & 0.62 & 0.24 & 0.45 & \multicolumn{1}{c|}{0.57} & 0.21 & 0.04 & 0.01 & 0.14 & 0.05 \\
MC-SAT & 0.43 & 0.63 & 0.24 & 0.46 & \multicolumn{1}{c|}{0.57} & 0.13 & 0.04 & 0.03 & 0.11 & 0.08 \\
HL-MRF & 0.60 & 0.78 & 0.52 & 0.70 & \multicolumn{1}{c|}{0.81} & 0.26 & 0.18 & 0.06 & 0.27 & 0.19 \\
\bottomrule
\end{tabular}
\end{table}

In Section~\ref{sec:mln_infer} we compare ExpressGNN with five probabilistic inference methods under open-world semantics. This is different from the original works, where they generally adopt the closed-world setting due to the scalability issues. More specifically, the original works assume that the predicates (except the ones in the query) observed in the knowledge base is \textit{closed}, meaning for all instantiations of these predicates that do not appear in the knowledge base are considered \textit{false}. Note that only the query predicates remain open-world in this setting.

For sanity checking, we also conduct these experiments with a closed-world setting. We found the results summarized in Table~\ref{table:close_table} are close to those reported in the original works. This shows that we have a fair setup (including memory size, hyperparameters, etc.) for those competitor methods. Additionally, one can find that the AUC-PR scores compared to those (Table~\ref{table:inference}) under open-world setting are actually better. This is due to the way the datasets were originally collected and evaluated generally complies with the closed-world assumption. But this is very unlikely to be true for real-world and large-scale knowledge base such as Freebase and WordNet, where many \textit{true} facts between entities are not observed. Therefore, in general, the open-world setting is much more reasonable, which we follow throughout this paper.

\section{Logic Formulae}
\label{app:rules}

We list some examples of logic formulae used in four benchmark datasets in Table~\ref{tab:all_rule}. The full list of logic formulae is available in our source code repository. Note that these formulae are not necessarily as clean as being always true, but are typically true.

For UW-CSE and Cora, we use the logic formulae provided in the original dataset. UW-CSE provides 94 hand-coded logic formulae, and Cora provides 46 hand-coded rules.  For Kinship, we hand-code 22 first-order logic formulae. For FB15K-237, we first use Neural LP~\citep{yang2017differentiable} on the full data to generate candidate rules. Then we select the ones that have confidence scores higher than 90\% of the highest scored formulae sharing the same target predicate. We also de-duplicate redundant rules that can be reduced to other rules by switching the logic variables. Finally, we have generated 509 logic formulae for FB15K-237.

\setlength\tabcolsep{4pt}
\begin{table}[h!]
\centering
\caption{Examples of logic formulae used in four benchmark datasets.}
\label{tab:all_rule}
\resizebox{\textwidth}{!}{
\begin{tabular}{@{}ll@{}}
\toprule
Dataset & \multicolumn{1}{l}{First-order Logic Formulae} \\ \midrule
\multirow{5}{*}{Kinship} & $\texttt{Father(X,Z)} \land \texttt{Mother(Y,Z)} \Rightarrow \texttt{Husband(X,Y)}$ \\
 & $\texttt{Father(X,Z)} \land \texttt{Husband(X,Y)} \Rightarrow \texttt{Mother(Y,Z)}$ \\
 & $\texttt{Husband(X,Y)} \Rightarrow \texttt{Wife(Y,X)}$ \\
 & $\texttt{Son(Y,X)} \Rightarrow \texttt{Father(X,Y)} \lor \texttt{Mother(X,Y)}$ \\
 & $\texttt{Daughter(Y,X)} \Rightarrow \texttt{Father(X,Y)} \lor \texttt{Mother(X,Y)}$ \\ \midrule
\multirow{5}{*}{UW-CSE} & $\texttt{taughtBy(c, p, q)} \land \texttt{courseLevel(c, Level500)} \Rightarrow \texttt{professor(p)}$ \\
 & $\texttt{tempAdvisedBy(p, s)} \Rightarrow \texttt{professor(p)}$ \\
 & $\texttt{advisedBy(p, s)} \Rightarrow \texttt{student(s)}$ \\
 & $\texttt{tempAdvisedBy(p, s)} \Rightarrow \texttt{student(s)}$ \\
 & $\texttt{professor(p)} \land \texttt{hasPosition(p, Faculty)} \Rightarrow \texttt{taughtBy(c, p, q)}$ \\ \midrule
\multirow{5}{*}{Cora} & $\texttt{SameBib(b1,b2)} \land \texttt{SameBib(b2,b3)} \Rightarrow \texttt{SameBib(b1,b3)}$ \\
 & $\texttt{SameTitle(t1,t2)} \land \texttt{SameTitle(t2,t3)} \Rightarrow \texttt{SameTitle(t1,t3)}$ \\
 & $\texttt{Author(bc1,a1)} \land \texttt{Author(bc2,a2)} \land \texttt{SameAuthor(a1,a2)} \Rightarrow \texttt{SameBib(bc1,bc2)}$ \\
 & $\texttt{HasWordVenue(a1, +w)} \land \texttt{HasWordVenue(a2, +w)} \Rightarrow \texttt{SameVenue(a1, a2)}$ \\
 & $\texttt{Title(bc1,t1)} \land \texttt{Title(bc2,t2)} \land \texttt{SameTitle(t1,t2)} \Rightarrow \texttt{SameBib(bc1,bc2)}$ \\ \midrule
\multirow{5}{*}{FB15K-237} & $\texttt{position(B, A)} \land \texttt{position(C, B)} \Rightarrow \texttt{position(C, A)}$ \\
 & $\texttt{ceremony(B, A)} \land \texttt{ceremony(C, B)} \Rightarrow \texttt{categoryOf(C, A)}$ \\
 & $\texttt{film(B, A)} \land \texttt{film(C, B)} \Rightarrow \texttt{participant(A, C)}$ \\
 & $\texttt{storyBy(A, B)} \Rightarrow \texttt{participant(A, B)}$ \\
 & $\texttt{adjoins(A, B)} \land \texttt{country(B, C)} \Rightarrow \texttt{serviceLocation(A, C)}$ \\ \bottomrule
\end{tabular}
}
\end{table}

\end{document}